\documentclass{article}

\PassOptionsToPackage{square,numbers}{natbib}
  \usepackage[preprint]{neurips_2026}

\usepackage[utf8]{inputenc} 
\usepackage[T1]{fontenc}    
\usepackage{hyperref}       
\usepackage{url}            
\usepackage{booktabs}       
\usepackage{amsfonts}       
\usepackage{amsmath}
\usepackage{amsthm}
\newtheorem{theorem}{Theorem}
\newtheorem{lemma}{Lemma}
\usepackage{nicefrac}       
\usepackage{microtype}      
\usepackage{xcolor}         
\usepackage{graphicx}
\title{Euclidean Embedding of Data Using Local Distances}

\author{%
  Dimitris~Arabadjis
\\
  Department of Statistics and Actuarial-Financial Mathematics\\
  University of the Aegean\\
  Karlovassi, Samos, PA 83200 \\
  \texttt{darampatzis@aegean.gr} \\
}

\begin{document}

\maketitle

\begin{abstract}
We study the problem of recovering a globally consistent Euclidean embedding of data, given only a local distance graph and propose a method that optimally represents these distances. The method operates solely on a neighborhood graph weighted by pairwise distances, without requiring any prior vector representation of the data. The embedding is obtained by solving a variational problem that matches local, on‑graph distances to the Euclidean metric, induced by the differentials of the embedding functions. The resulting Euler–Lagrange equations are derived in a coordinate‑free form, enabling direct evaluation of all operators from the distance graph alone. Though non-linear and missing an explicit expression for their non-linearity, these equations are shown to be resolved as an iteratively updated sparse linear problem. The main contributions of the proposed approach are (a) the derivation of the functional equations governing the optimal Euclidean embedding in the continuum, (b) a representation‑free formulation that requires only a neighborhood distance graph and no feature vectors and (c) an estimation procedure based exclusively on local graph operations. We experimentally evaluate the resulting non‑parametric algorithm on synthetic manifolds and real datasets, demonstrating consistent preservation of local metric structure and neighboring relations, while approximating the global isometric embedding.
\end{abstract}

\section{Introduction}
\label{Sct_intro}
The intention of this work is to introduce an approach to the problem of embedding data into Euclidean vector spaces, based on rigorous optimality guarantees (Sect. \ref{Sct_mthd}), while retaining generality and applicability. Specifically, the developed methodology avoids any reference to prior representations of the considered data, other than their distance graph.  Moreover, the operations applied remain local, processing and producing sparse data.

Actually, the driving aim of the proposed approach is to approximate the globally optimal Euclidean embedding that IsoMap \citep{ref_isomap} achieves via Multidimensional Scaling (MDS) \citep{ref_mds} on the full matrix of the datapoints' pair-wise geodesic distances, using only local distance data.

The methodological gap related with this aim is identified to the fact that methods which operate on local data, though capable to retain these data in the computed embedding, they fail to lift them globally and they distort the geodesic distances of the original data manifold. This is not a methodological deficiency but an implication of the optimization problems that are resolved in these methods.  Specifically, according to the Minimum-Distortion prototype, introduced in \citep{ref_mde}, many of the existing local embedding methods resolve, point-wise, a(n) (un)constrained optimization problem formulated in terms of the embedding's local distances. Within this framework, even if one asks that these distances match the true data distances, the approximation holds only locally, due to the fact that the embedding is computed pointwise, without evaluating that there should be a (continuous/smooth) mapping underlying this embedding. Consequently, if pointwise optimization is not lifted, only the data distances involved in the optimization procedure could be approximated by the embedding's  Euclidean distances, thus rendering the global embedding methods unavoidable.

The proposed approach drops the pointwise formulation of the embedding problem and uses functional optimization to determine the optimal mapping of the data manifold in the Euclidean space that locally approximates data distances. Then, integrability of the local approximations lifts this correspondence globally, without incorporating the actual shortest path distances between non-neighboring datapoints.

Local embedding methods that fall within the Minimum-Distortion prototype mainly split in two classes; the \textit{spectral methods} that identify the embedding with the null space of the linear evaluator of the data manifold features that should be retained and the \textit{stochastic methods} that iteratively construct embeddings, whose local Euclidean distances induce a stationary distribution equivalent to the non-stationary distribution induced by the original data distances. 

\textit{Spectral methods}: a) Locally Linear Embedding (LLE) \citep{ref_lle} that locally matches the metric induced by the data distances and the metric of the embedding (see also \citep{ref_lle1}), b) Local Tangent Space Alignment (LTSA) \citep{ref_ltsa} that locally matches the eigenvectors of these metrics, c) Laplacian Eigenmaps \citep{ref_laplacian} that detects the embedding of the minimal total variation, d) Hessian Eigenmaps \citep{ref_hess} that detects the embedding that is in the kernel of the Hessian operator

\textit{Stochastic methods}: a) t-SNE \citep{ref_tsne} is the source point of the whole class and it is based on KL-divergence minimization between the embedding and data distributions, b) LargeVis \citep{ref_largevis} modifies the distributions' comparison, incorporating non-local points' relationships and c)  Uniform Manifold Approximation and Projection (UMAP) \citep{ref_umap} employs fuzzy operations to symmetrize datapoints' similarity graph and cross-entropy to symmetrize distributions' comparison.

Concerning the optimality of the determined embedding, spectral methods determine the global minimizers of the considered quadratic optimization problems, while the non-linearity of the stationarity equations of the stochastic methods' probabilistic objectives, leads to iterative (stochastic) gradient descent approximations. The stationarity equations determined under the proposed approach are highly non-linear (Sect. \ref{Sct_ELeqs}). However, the iterative scheme developed to resolve them (Sct. \ref{Sct_Iter}) provably converges to a global minimizer of the approximation error. Moreover, these stationarity equations are functional, not point-wise algebraic ones, naturally interpreting the point-wise embedding results as evaluations of the embedding functions. So, though these equations are properly discretized on graphs (Sct. \ref{Sct_impl}) to obtain pointwise evaluations of the embedding, the integrability of these evaluations is intrinsic in the equations' formulas and their discrete realization on graphs, thus lifting the sought local isometric approximations to a global approximation.

\section{Graph-based formulation of the Euclidean embedding problem}
\label{Sct_brief}
The core idea underlying the proposed  methodology is that a smooth embedding, whose Euclidean distances optimally approximate the corresponding geodesic distances on the data manifold, is a vector of scalar fields whose differentials optimally align with infinitesimal, on-manifold transports that trivialize the manifold's metric (they turn it into the identity matrix). While the directions of such transports can always be found via spectral decomposition of the metric, only flat manifolds allow for isometrically deform them into integrable entities. The determination of the closest integrable entities to metric-trivializing infinitesimal transports is the problem considered in the proposed methodology.

Technically, this problem is formulated and resolved in terms of differential 1-forms, which are the integrands of line integrals, thus formally representing infinitesimal transports. However, these entities and their associated operators that resolve the problem have immediate analogues on graphs, which  are analytically derived in appendix \ref{appA} and outlined in the following Scts. \ref{Sct_mds} and \ref{Sct_dffrms_grphs}.

\subsection{Local Euclidean frames from distances}
\label{Sct_mds}
Let $\mathcal{G}$ be a distance graph, namely an undirected graph, weighted by the connected nodes' mutual distances. While centered on an arbitrary vertex, say $v_i$, if the shortest path (geodesic) distances between its neighbors were available, then one could construct a Euclidean embedding for this neighborhood, by locally applying multidimensional scaling (Gram matrix computation from distances and spectral decomposition).  Specifically, if $M^{v_i}$ is the matrix of squared geodesic distances between the neighbors of $v_i$, then the Gram matrix, $G^{v_i}$ of the underlying Euclidean representation of the neighborhood, having $v_i$ at the origin, reads 
\begin{equation*}
\left.G^{v_i}\right|_{v_j,v_k}=1/2(\left.M^{v_i}\right|_{v_i,v_j} + \left.M^{v_i}\right|_{v_k,v_i} - \left.M^{v_i}\right|_{v_j,v_k})
\end{equation*}
where $v_j,v_k$ are neighbors of $v_i$.  Then, the eigen-decomposition $G^{v_i}=V^{v_i}S^{v_i}(V^{v_i})^T$, recovers a best approximating $N$-dimensional Euclidean embedding of the neighborhood, centered at $v_i$, as
\begin{equation}
\label{eq_LclFrms}
E^{v_i}=V^{v_i}(S^{v_i})^{1/2}\mathbf{1}_N
\end{equation}
where $\mathbf{1}_N$ denotes a (\# of neighbors) - by - $N$ identity matrix, selecting the largest $N$ eigenvalues of the Gram matrix. Actually, $E^{v_i}$ is a matrix whose rows correspond to the vertices of the neighbors of $v_i$ and its rows to the embedding coordinates, representing the geodesics from $v_i$ to its neighbors as straight lines in $\mathbb{R}^N$. However, geodesics between the neighbors of $v_i$ are not necessarily straightened, since this straightening is controlled by these vertices' neighborhoods.

For this reason, $E^{v_i}$ is not computed on the basis of the subgraph $\Gamma(v_i)$ of the immediate neighbors of $v_i$ but on the subgraph $\Gamma^2(v_i)$ of the second order neighborhood of $v_i$. With such an $E^{v_i}$, except of its rows that straighten the paths between $v_i$ and the vertices of $\Gamma^2(v_i)$, rows' differences $\left. E^{v_i}\right|_{v_k} - \left. E^{v_i}\right|_{v_j}$, if $v_j,v_k \in \Gamma(v_i)$, also straighten geodesic paths between these vertices. However, computation of $E^{v_i}$ for the vertices of $\Gamma^2(v_i)$ requires valid geodesic distances between them, which in turn, require the subgraph of all neighborhoods of the vertices of $\Gamma^2(v_i)$ that is actually the subgraph $\Gamma^3(v_i)$ of the third order neighborhood of $v_i$.

Consequently, in short, for each vertex $v_i$ and for a constant local embedding dimension:
\begin{enumerate}
\item The subgraphs $\Gamma (v_i)$, $\Gamma^2(v_i)$ and $\Gamma^3(v_i)$ of the first, second and third order neighborhoods of $v_i$ are correspondingly considered,
\item Geodesic distances between the vertices of $\Gamma^2(v_i)$ are computed via shortest paths in $\Gamma^3(v_i)$ 
\item The Gram matrix $G^{v_i}$ is computed on the basis of these distances and its rank - $N$ spectral decomposition offers $E^{v_i}$ via \eqref{eq_LclFrms}
\item Only the rows of $E^{v_i}$ that correspond to vertices of $\Gamma(v_i)$ are maintained.
\end{enumerate}
\subsection{Metric-trivializing transports and functions' differentials on graphs}
\label{Sct_dffrms_grphs}
The geodesics' straightening frames $E^{v_i}$ allow for representing points within the convex hull of the neighborhood of each vertex $v_i$, via barycentric coordinates $\lambda^{v_i}$. Namely, for any such point $p$ its $v_i $ - centered Euclidean representation $x^{v_i}(p)$ reads
\begin{equation}
\label{eq_coords}
x^{v_i}(p)=\sum_{v_j\in\Gamma(v_i)}{\left.\lambda^{v_i}\right|_{v_j}(p)\left.E^{v_i}\right|_{v_j}}
\end{equation}
Using these coordinates for $p$, if $p$ lies along the shortest paths of $\Gamma^2(v_i)$ connecting the vertices of $\Gamma(v_i)$, the infinitesimal transports $dx(p)$ correspond to transports along geodesics, thus trivializing the metric. Now, if $p$ lies anywhere else in the convex hull of $\Gamma(v_i)$, the identity matrix is a second order approximation of the metric of $dx^{v_i}(p)$ and specifically, with an $O(\underset{v_j\in\Gamma(v_i)}{\min}|x^{v_i}(p)-x^{v_i}(v_j)|^2)$ error. Consequently, up to second order approximation, the infinitesimal transport
 \begin{equation}
\label{eq_alpha}
\alpha^{v_i}(p)=\sum_{v_j\in\Gamma(v_i)}{\left.d\lambda^{v_i}\right|_{v_j}(p)\left.E^{v_i}\right|_{v_j}}
\end{equation}
can be thought as trivializing the metric.

If the $N$ components of this transport were integrable, then one could recover a global Euclidean isometry of the data manifold in $\mathbb{R}^N$. In order to do so approximately, we should determine $N$ functions $\varphi=[\varphi_i]_{i=1..N}$, whose differential $d\varphi$ best approximates the trivializing transport, up to orthogonal $N$-dimensional transformations. The realization of the functions' differential on a distance graph, in terms of the local coordinates of \eqref{eq_coords} is derived in appendix \ref{appA}, thus allowing to approximate $d\varphi$ , again up to the second order, by
  \begin{equation}
\label{eq_diff}
d\varphi^{v_i}(p)=\sum_{v_j\in\Gamma(v_i)}{\left.d\lambda^{v_i}\right|_{v_j}(p) (\varphi(v_j)-\varphi(v_i))}
\end{equation}
\subsection{The quadratic approximation error}
Having realized the metric trivializing transports $\alpha$ and the functions' differential $d\varphi$ on distance graphs, the next step is to formulate the error of approximating $\alpha$ by $d\varphi$ and evaluate it on distance graphs. Specifically, let
 \begin{equation}
\label{eq_error}
\varepsilon[\varphi,Q]=d\varphi-\alpha Q
\end{equation}
the error of approximating the trivializing transports $\alpha$ by functions' differentials $d\varphi$, up to orthogonal transformations $Q$. Actually, $\varepsilon$ is an infinitesimal transport between a $Q$ - aligned version of the metric - trivializing $\alpha$ and its integrable approximate $d\varphi$. The quadratic measure of $\varepsilon$ is formally offered by the inner product, as defined for differential forms, within the Hodge theory (e.g. refer to the ch. 0, par. 6 of \citep{hodge}). In  appendix \ref{appA}, it is shown that using on - graph realizations \eqref{eq_alpha} and \eqref{eq_diff} for $\alpha$ and $d\varphi$, respectively, this quadratic error reads
\begin{subequations}
\label{eq_L2error}
\begin{align}
\|\varepsilon[\varphi,Q]\|^2 &= \sum_{v_i}{\sum_{v_j , v_k\in \Gamma(v_i)}{\left.\Lambda^{v_i}\right|_{v_j , v_k}\left\langle\left.\mathcal{E}^{v_i}\right|_{v_j} , \left.\mathcal{E}^{v_i}\right|_{v_k}\right\rangle_{\mathbb{R^N}}} \mathrm{vol}(\Gamma(v_i))}
\\
\left.\mathcal{E}^{v_i}\right|_{v_j}  &\equiv \varphi(v_j)-\varphi(v_i)-\left.E^{v_i}\right|_{v_j} Q
\\
\Lambda^{v_i} &: (E^{v_i})^T \Lambda^{v_i} E^{v_i}=I_{N\times N}
\label{eq_Lambda}
\end{align}
\end{subequations}
Actually, $\mathcal{E}^{v_i}$ evaluates the failure of the local metric - trivializing frames $E^{v_i}$ to be aligned with vertex functions' differences within graph's neighborhoods, while $\Lambda^{v_i}$ is the rank - $N$ pseudo-inverse of the Gram matrix $G^{v_i}$, properly weighting the on-graph inner product of $\mathcal{E}^{v_i}$ with itself.

\section{The embedding method}
\label{Sct_mthd}
\subsection{The stationarity equations of the optimal Euclidean embedding}
\label{Sct_ELeqs}
In order to formulate and resolve the optimal Euclidean embedding problem in a general setting, the error functional $\|\varepsilon[\varphi,Q]\|^2$ is considered independently of the on - graph realizations \eqref{eq_alpha} and \eqref{eq_diff} of $\alpha$ and $d\varphi$, within the framework of differential forms. Then, allowing $\varphi$ and $Q$ to infinitesimally vary from their error minimizing counterparts $\tilde{\varphi}$ and $\tilde{Q}$, one gets the Euler - Lagrange equations that render the error functional stationary. This derivation is made in appendix \ref{appB} and within the manifold to be embedded, $\Omega$, the stationary $\tilde{\varphi}$ and $\tilde{Q}$ are given by the equations
\begin{subequations}
\label{eq_ELeqs}
\begin{align}
&\mathrm{div}(d\tilde{\varphi}) = \mathrm{div}(\alpha \tilde{Q}) ~, ~~ \text{with boundary conditions} ~~ \left.d\tilde{\varphi}\right|_{\partial\Omega} = \left.\alpha \tilde{Q}\right|_{\partial\Omega} 
\label{eq_ELphi}
\\
&\tilde{Q} = R_{\tilde{\varphi}} {L_{\tilde{\varphi}}}^T ~, ~~ \text{where} ~ (L_{\tilde{\varphi}} , \Sigma_{\tilde{\varphi}} , R_{\tilde{\varphi}} ) = \mathrm{svd} (B_{\tilde{\varphi}}) ~~ \text{for} ~ B_{\tilde{\varphi}} : \langle d\tilde{\varphi}_i , \alpha_j \rangle =[B_{\tilde{\varphi}}]_{ij} d\mathrm{vol} 
\label{eq_ELq}
\end{align}
\end{subequations}

To clarify how these equations lead to the determination of $\tilde{\varphi}$ and $\tilde{Q}$, we should first recall that $d\varphi$ and $\alpha$ contain $N$ infinitesimal transports on $\Omega$, arranged in $1\times N$ matrices and that $Q$ is the $N\times N$ orthogonal matrix that is free to properly align $\alpha$ with $d\varphi$, without modifying the metric - trivializing identity of $\alpha$. Then, \eqref{eq_ELq} indicates that $\tilde{Q}$, which best aligns $\alpha$ with a given $d\varphi$ results from the left and right singular vectors of the matrix of the integrands of the inner products between all components of $d\varphi$ and $\alpha$. Actually, this matrix contains the directional derivatives of the components of $\varphi$ along the directions of the infinitesimal transports of $\alpha$. In turn, \eqref{eq_ELphi} indicates that the elements of $\tilde{\varphi}$ result from corresponding Neumann PDE problems that enforce their differentials to differ from the properly aligned $\alpha$ by divergence free transports. Actually, on the left hand side of \eqref{eq_ELphi} one can identify the on-manifold Laplacian of the elements of $\tilde{\varphi}$, while on the right hand side there is the on-manifold divergence of the directions of the infinitesimal transports of $\alpha$, transformed by the $\tilde{\varphi}$ - dependent $\tilde{Q}$.

Consequently, the Euler-Lagrange equations that govern the optimal Euclidean embedding $\tilde{\varphi}$ become coupled, with non-linear and non-explicit terms with respect to the derivatives of $\tilde{\varphi}$, introduced by the singular value decomposition of \eqref{eq_ELq}. Thus, $\tilde{\varphi}$ cannot be determined in one step, since its equation lacks an explicit formula. In order to overcome this deadlock, an alternating procedure is adopted, which, at its fixed point, determines $\tilde{\varphi}$ and $\tilde{Q}$, simultaneously.

\subsection{Iterative determination of the embedding}
\label{Sct_Iter}
In order to break the interdependency of the approximation error minimizers $\tilde{\varphi}$ and $\tilde{Q}$, a plausible approach is to consider $\tilde{Q}$ fixed while computing $\tilde{\varphi}$ by \eqref{eq_ELphi}, then use this $\tilde{\varphi}$ to re-compute $\tilde{Q}$ by \eqref{eq_ELq}, return to re-estimate $\tilde{\varphi}$ and so on. This iterative scheme is summarized by the following recursion, beginning from $\tilde{Q}|_0=I $
\begin{subequations}
\label{eq_ELiter}
\begin{align}
&\left.\tilde{\varphi}\right|_{n+1} : 
\begin{cases}
\mathrm{div}(d\left.\tilde{\varphi}\right|_{n+1}) = \mathrm{div}\left(\alpha \tilde{Q}|_n\right) \\ 
\left.(d\left.\tilde{\varphi}\right|_{n+1})\right|_{\partial\Omega} = \left.\left(\alpha \tilde{Q}|_n\right)\right|_{\partial\Omega}
\end{cases} 
\label{eq_ELiter_phi}
\\
&\tilde{Q}|_{n+1}= R_{\tilde{\varphi}|_{n+1}} {L_{\tilde{\varphi}|_{n+1}}}^T
\label{eq_ELiter_q}
\end{align}
\end{subequations}
where $L_{\tilde{\varphi}|_{n++1}}$ an  $R_{\tilde{\varphi}|_{n+1}}$ result by the SVD of \eqref{eq_ELq}, applied on $d\tilde{\varphi}|_{n+1}$ .

In appendix \ref{appC} it is shown that this iterative scheme converges to the global minimizers of the error functional $\|\varepsilon[\varphi,Q]\|^2$, according to the following theorem.
\begin{theorem}
\label{th_converge}
Within an $N$-dimensional manifold $\Omega$, the error functional $\|\varepsilon[\varphi,Q]\|^2=\int_{\Omega}{\langle \varepsilon[\varphi,Q] , \varepsilon[\varphi,Q]\rangle}$, with $\varepsilon[\varphi,Q]$ defined by \eqref{eq_error} and with $\langle ,\rangle$ denoting the inner product of differential forms, the iterative process defined in \eqref{eq_ELiter} converges to the equivalence class $\left[\left.\tilde{\varphi}\right|_{\infty}, \tilde{Q}|_{\infty}\right]$ of  the global minimizers of $\|\varepsilon[\varphi,Q]\|^2$, which differ only by constant (rigid-body) orthogonal transformations and constant translations of $\tilde{\varphi}_\infty$.
\end{theorem}
Consequently, instead of the non-linear and explicitly intractable equation \eqref{eq_ELphi}, one can resolve the linear, Poisson-like equation of  \eqref{eq_ELiter_phi}, recursively aligning its right-hand side with the acquired solutions. Moreover, if global (rigid-body) orthogonal procrustes alignment of $\alpha$ to $d\left.\tilde{\varphi}\right|_{n}$ is removed from the local alignment applied by  $\tilde{Q}|_{n}$ and  $\tilde{\varphi}|_{n+1}$ is normalized to zero mean value, theorem \ref {th_converge} ensures fixed-point convergence of the iterative scheme to a single pair of optimal $\tilde{\varphi}|_{\infty}, {Q}|_{\infty}$. 

\section{Implementation of the method on distance graphs}
\label{Sct_impl}
While $\alpha$ and $d\varphi$ have explicit, on-graph realizations, given by \eqref{eq_alpha} and \eqref{eq_diff}, their inner product and their divergence, lack such realizations and they are necessary for evaluating \eqref{eq_ELiter_q} and \eqref{eq_ELiter_phi}, respectively. These realizations are given in the following Sct. \ref{Sct_impl_div}, thus allowing for the computational implementation of the embedding method derived in Sct. \ref{Sct_Iter}, which is given in Sct. \ref{Sct_impl_alg}

\subsection{Infinitesimal transports}
Formula \eqref{eq_alpha}, derived for the metric-trivilizing transports actually sets the prototype for representing any infinitesimal transport, on distance graphs. So, let $\zeta$ be such a transport, then for any point $p$ in the neighborhood $\Gamma(v_i)$ of a certain vertex $v_i$ of the graph, $\zeta$ may be written as 
\begin{equation}
\label{eq_diff1forms}
\zeta^{v_i}(p)=\sum_{v_j\in\Gamma(v_i)}{\left.d\lambda^{v_i}\right|_{v_j}(p)\left.Z^{v_i}\right|_{v_j}}
\end{equation}
with $\lambda^{v_i}$ induced by the local Euclidean representation \eqref{eq_coords} of $\Gamma(v_i)$. Consequently, an arbitrary infinitesimal transport $\zeta$ is represented by its edge function $Z$, which in turn, is computationally realized by a sparse matrix in the same filling pattern as the underlying graph's adjacency matrix.

Moreover, definition \eqref{eq_alpha} for the metric trivializing transports $\alpha$ reveals the necessity for a computational representation of vectors of infinitesimal transports, which, on-graph, correspond to vectors of edge functions that computationally, are represented by tuples of sparse matrices,  following the filling pattern of the graph's adjacency matrix. So, the local frames $E^{v_i}|_{v_j}$ that determine the on-graph realization of $\alpha$, as determined in Sct. \ref{Sct_mds}, they correspond to $N$ edge functions (the coordinates of the Euclidean embedding of $\Gamma(v_i)$ in $\mathbb{R}^N$), which are computationally represented as a tuple, \texttt{Frm}, of $N$ sparse matrices, \texttt{F}, in the same filling pattern as the distance graph's adjacency matrix. This representation is identified by the correspondence
\begin{equation}
\label{eq_frmE}
E^{v_i}|_{v_j}\equiv\left[\left(\text{\texttt{Frm}}(k).~\text{\texttt{F}}\right)^{v_i}|_{v_j}\right]_{k=1...N}
\end{equation} 

\subsection{Inner product, divergence and Laplacian}
\label{Sct_impl_div}
The key operation, whose on-graph realization enables the implementation of the iterative embedding scheme of \eqref{eq_ELiter} is the inner product of the infititesimal transports .  In appendix \ref{appA}, using the representation of \eqref{eq_diff1forms}, the on-graph realization of the inner product between two arbitrary transports $\zeta$ and $\eta$ is derived. Moreover, there is given the evaluation of the corresponding formula as a vector product of $N$ edge functions. The constituent elements of this evaluation are a) the sparse matrices $Z,H$ that represent $\zeta,\eta$, respectively and b) the $N$ edge functions ${E_{-}}^{v_i}|_{v_j}\in \mathbb{R}^N$ that result from the local frames $E^{v_i}$ of \eqref{eq_LclFrms} as
\begin{equation}
\label{eq_LclFrmsInv}
{E_{-}}^{v_i}|_{v_j} = \sqrt{\mathrm{vol}(\Gamma(v_i))} V^{v_i}|_{v_j}(S^{v_i})^{-1/2}\mathbf{1}_N
\end{equation}
where, again, $\mathbf{1}_N$ denotes a (\# of neighbors) - by - $N$ identity matrix, selecting the largest $N$ eigenvalues of $S^{v_i}$ (please refer to \eqref{eq_LclFrms}) and $\mathrm{vol}(\Gamma(v_i))$ denotes the volume of the $N$-dimensional convex hull of $\Gamma(v_i)$. In correspondence with the computational representation \eqref{eq_frmE} of $E^{v_i}|_{v_j}$, ${E_{-}}^{v_i}|_{v_j}$ is represented as a tuple, \texttt{FrmInv}, of $N$ edge functions \texttt{F}, via the correspondence
\begin{equation}
\label{eq_frmInvE}
{E_{-}}^{v_i}|_{v_j}\equiv\left[(\text{\texttt{FrmInv}}(k).~\text{\texttt{F}})^{v_i}|_{v_j}\right]_{k=1...N}
\end{equation} 
On the basis of these data, the inner product of the arbitrary infinitesimal transports $\zeta$ and $\eta$ is implemented as
\begin{equation}
\label{eq_innrProdMat}
\langle\eta,\zeta\rangle=\sum_{k=1...N}{(\text{\texttt{sum}}\left(\text{\texttt{FrmInv}}(k).~\text{\texttt{F}} \odot H,\text{\texttt{dim}}=2\right) \odot (\text{\texttt{sum}}\left(\text{\texttt{FrmInv}}(k).~\text{\texttt{F}} \odot Z,\text{\texttt{dim}}=2\right)}
\end{equation}
where $\odot$ denotes the element-wise product of matrices and \texttt{sum} denotes the summation of matrix elements, along the dimension \texttt{dim}.

Using the on-graph realization of the inner product and \eqref{eq_diff} for  the functions' differential, in appendix \ref{appA}, the on-graph realization of the divergence operator is determined, as the inner product - dual of the differential. The computational evaluation of the operator follows from \eqref{eq_innrProdMat} as
\begin{multline}
\label{eq_div}
\mathrm{div}(\zeta)=\sum_{k=1...N}{{\text{\texttt{FrmInv}}(k).\text{\texttt{F}}}^T \text{\texttt{sum}}\left(\text{\texttt{FrmInv}}(k).\text{\texttt{F}} \odot Z,\text{\texttt{dim}}=2\right)}\\
-\sum_{k=1...N}{\text{\texttt{sum}}\left(\text{\texttt{FrmInv}}(k).\text{\texttt{F}} , \text{\texttt{dim}}=2\right) \odot \text{\texttt{sum}}\left(\text{\texttt{FrmInv}}(k).\text{\texttt{F}} \odot Z,\text{\texttt{dim}}=2\right)}
\end{multline}
Using \eqref{eq_div} and substituting in $\zeta$ a function's differential, $df$, as realized on distance graphs via \eqref{eq_diff}, one has $Z^{v_i}|_{v_j}=f|_{v_j}-f|_{v_i}$, thus obtaining the implementation of the Laplacian operator 
\begin{multline}
\label{eq_lplc}
\mathcal{L}=\sum_{k=1...N}{{\text{\texttt{FrmInv}}(k).\text{\texttt{F}}}^T \text{\texttt{FrmInv}}(k).\text{\texttt{F}}}
+\texttt{diag}(( \text{\texttt{sum}}\left(\text{\texttt{FrmInv}}(k).\text{\texttt{F}},\text{\texttt{dim}}=2\right))^2)\\
- \text{\texttt{sum}}\left(\text{\texttt{FrmInv}}(k).\text{\texttt{F}},\text{\texttt{dim}}=2\right) \odot  (\text{\texttt{FrmInv}}(k).\text{\texttt{F}}+{\text{\texttt{FrmInv}}(k).\text{\texttt{F}}}^T)
\end{multline}
where \texttt{diag} turns its input vector to a sparse diagonal matrix and the squaring evident in its argument is supposed to act element-wise. In both, \eqref{eq_div} and \eqref{eq_lplc} the element-wise matrix operations $+,-,\odot$ are considered to act as their vectorized implementation do, namely by operating column- or row-wise if one of their operands is a column or resp. row vector.

\subsection{The embedding algorithm}
\label{Sct_impl_alg}
\textbf{Input}: $A$ - the distance graph's adjacency matrix, $N$ - the dimension of the sought representation.

\textbf{Initialization}: Via the process of Sct. \ref{Sct_mds}, compute the local Euclidean frames $E^{v_i}$ and use them to estimate the volume $\mathrm{vol}(\Gamma(v_i))$ of the corresponding neighborhoods $\Gamma(v_i)$. Then, via \eqref{eq_LclFrmsInv}, compute ${E_{-}}^{v_i}$. Collect $E$ and $E_{-}$ correspondingly to the tuples \texttt{Frm} and \texttt{FrmInv}, each containing $N$ sparse matrices following the filling pattern of $A$. Using \texttt{FrmInv}, compute the Laplacian matrix $\mathcal{L}$, via \eqref{eq_lplc}. This is a sparse matrix following the filling pattern of $A^2$. Use $\mathcal{L}$ to compute a pre-conditioner (here a thresholded incomplete Cholesky factorization (ICT)) for the iterative sparse solver employed to invert it (here the Preconditioned Conjugate Gradient (PCG)). Initialize the $\text{\#-vertices}\times N\times N$ matrices $\tilde{Q}$, with $\tilde{Q}_{i,*,*}=I_{N\times N}$, for all $i=1...\text{\#-vertices}$ and \texttt{a}, with $\texttt{a}_{*,k,m} = \texttt{sum}\left(\texttt{FrmInv}(k).\texttt{F} \odot \texttt{Frm}(m).\texttt{F},\texttt{dim}=2\right)$.

\textbf{Recursion}: Compute the $\text{\#-vertices}\times N$ matrix \texttt{DaQ}, evaluating $\mathrm{div}(\alpha \tilde{Q})$, by applying \eqref{eq_div} in the form $\texttt{DaQ}_{*,m}=\sum_{k=1}^N{{\texttt{FrmInv}(k).\texttt{F}}^T \texttt{a}_{*,k,m}-\texttt{sum}\left(\texttt{FrmInv}(k).\texttt{F},\texttt{dim}=2\right)\odot\texttt{a}_{*,k,m}}$. Use a preconditioned iterative sparse solver to obtain the $\text{\#-vertices}\times N$ matrix $\tilde{\varphi}$ that results from the solutions of $\mathcal{L}\tilde{\varphi}_{*,m}=\texttt{DaQ}_{*,m}$, which evaluate \eqref{eq_ELiter_phi} on distance graphs. Use equation \eqref{eq_innrProdMat} to compute the $\text{\#-vertices}\times N\times N$ matrix $B_{\tilde{\varphi}}$ of \eqref{eq_ELq} in the form $[B_{\tilde{\varphi}}]_{*,m,n}=\sum_{k=1}^N{\left({\texttt{FrmInv}(k).\texttt{F}}^T \tilde{\varphi}_{*,m}-\texttt{sum}\left(\texttt{FrmInv}(k).\texttt{F},\texttt{dim}=2\right)\odot\tilde{\varphi}_{*,m}\right)\odot\texttt{a}_{*,k,n}}$. Apply SVD to each $N\times N$ matrix $[B_{\tilde{\varphi}}]_{i,*,*}$ and obtain via \eqref{eq_ELq} the corresponding elements $Q_{i,*,*}$ of the candidate update $Q$ of the matrix $\tilde{Q}$. Compute the rigid-body rotation $Q_0=R_0{L_0}^T$, resulting from the $(L_0,S_0,R_0)=\mathrm{svd}(\texttt{sum}\left(B_{\tilde{\varphi}},\texttt{dim}=1\right))$ and absorb it, setting $Q\leftarrow {Q_0}^T Q$. Evaluate the convergence error as $\texttt{err}={\underset{i}{\text{mean}}\left({\|{\tilde{Q}_{i,*,*}}^T{Q_{i,*,*}}-I\|_{2,2}}^2\right)}^{1/2}$ and break the recursion if $\texttt{err}$ falls below a predefined precision $\texttt{tol}$. Otherwise, update $\tilde{Q}_{i,*,*}\leftarrow \tilde{Q}_{i,*,*} Q_{i,*,*}$, $\texttt{a}_{i,*,*}\leftarrow\texttt{a}_{i,*,*}Q_{i,*,*}$ and iterate.

\textbf{Output}: $\tilde{\varphi}$ - the embedding of the graph's vertices, $\tilde{Q}$ - the orthogonal alignment matrices

\textbf{Complexity}: \textit{Memory complexity} is set by the non-zero elements of the Laplacian \eqref{eq_lplc}, which is $O(\text{\# - vertices})$. The computation of $\mathrm{vol}(\Gamma(v_i))$ could dominate this, if convex hull computations were held for large $N$. So, for $N\ge 7$, Kubota's projection formula is applied to 7-dimensional orthogonal projections of  $\Gamma(v_i)$. \textit{Time complexity}  of the recursion is set by the preconditioned iterative solution of the the sparse system that determines $\tilde{\varphi}$, which is $O((\text{\# - vertices})^{1+\epsilon})$, with $\epsilon\in[0,1/N)$,  controlled by the preconditioner.

\section{Experimental evaluation of the methodology}
\label{Sct_exprmnt}
\begin{table}
  \caption{Evaluation of synthetic datasets' representations' quality}
  \label{tab_swiss}
  \centering
  \begin{tabular}{p{1.2cm}p{0.7cm}p{0.7cm}p{0.8cm}p{0.7cm}p{1cm}p{0.8cm}p{0.8cm}p{0.8cm}p{0.7cm}p{1.1cm}}
    \toprule
    &\multicolumn{5}{c}{Klein's bottle}  & \multicolumn{5}{c}{Torus}  \\
    \cmidrule(r){2-11}
     &Lcl-mtrc  &Lcl-F1 &Glbl-Pearson &Glbl-mtrc &Glbl-prm &Lcl-mtrc  &Lcl-F1 &Glbl-Pearson &Glbl-mtrc &Glbl-prm  \\
    \midrule
    LLE & 0.30  & 0.59 &0.66 &1.00 & 0.09 & 0.45  & 0.50 &0.69 &0.98& 0.61\\
    HLLE & 3.88  & 0.23 &0.07 &0.90 & 0.80& 13.5 & 0.16 &0.06 &1.25& 0.97 \\
    LTSA &0.30  & 0.16 &0.07 &1.00 & 1.00 & 0.52 & 0.19 &0.05 &0.98& 1.00 \\
    LplcMps & 0.30  & 0.61 &0.67 &1.00 & 0.09 & 0.44  & 0.47 &0.77 &0.99& 0.67   \\
    DffMps & 0.27  & 0.86 &0.76  &0.75 & 0.03 & 0.83  & 0.56 &0.87 &0.25& 0.45   \\
    IsoMap & 0.77  & 0.81 &0.98  &0.09 & 0.00 & 0.60  & 0.55 &0.94  &0.14 & 0.67  \\
    proposed & 0.44  & 0.91  &0.95  &0.23 & 0.07 & 1.14  & 0.76  &0.80  &0.41 & 0.48\\
    \bottomrule
  \end{tabular}
\end{table}
\begin{figure}
  \centering
 \includegraphics[width=1\linewidth]{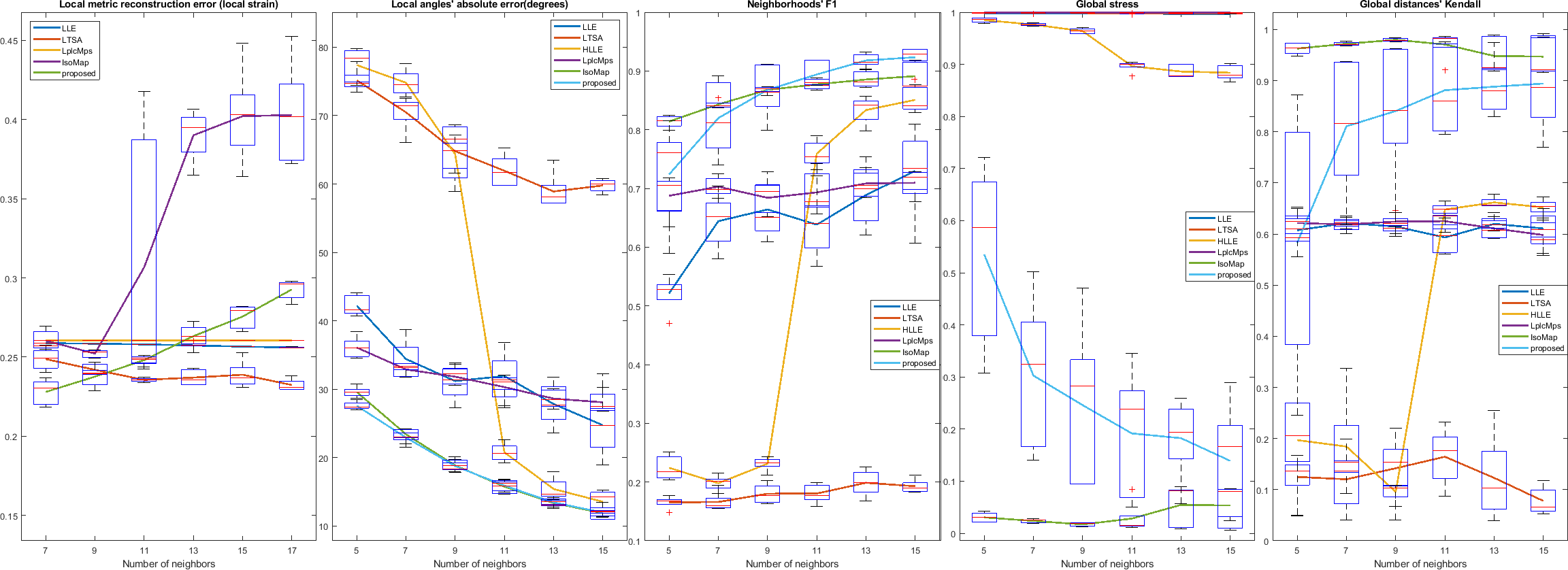}
  \caption{Swiss roll unfolding experiment for varying neighborhood size. The measures' box plots depict the maximum/minimum boundaries at whiskers, the boxes of the populations' 25th-75th percantile and the median values in red. The plotted lines depict the variation of the metrics' mean values with respect to the neighborhood size. At each neighborhood size, 15 random samplings were performed, each of 1500 datapoints. The proposed method is the only local one, approximating IsoMap in the metrics of global consistency (Global stress that IsoMap minimizes and Kendall correlation between the Euclidean and the original geodesic distances), while retaining local consistency of neighborhoods' representation.}
\label{fig_swiss}
\end{figure}
\begin{figure}
  \centering
 \includegraphics[width=1\linewidth]{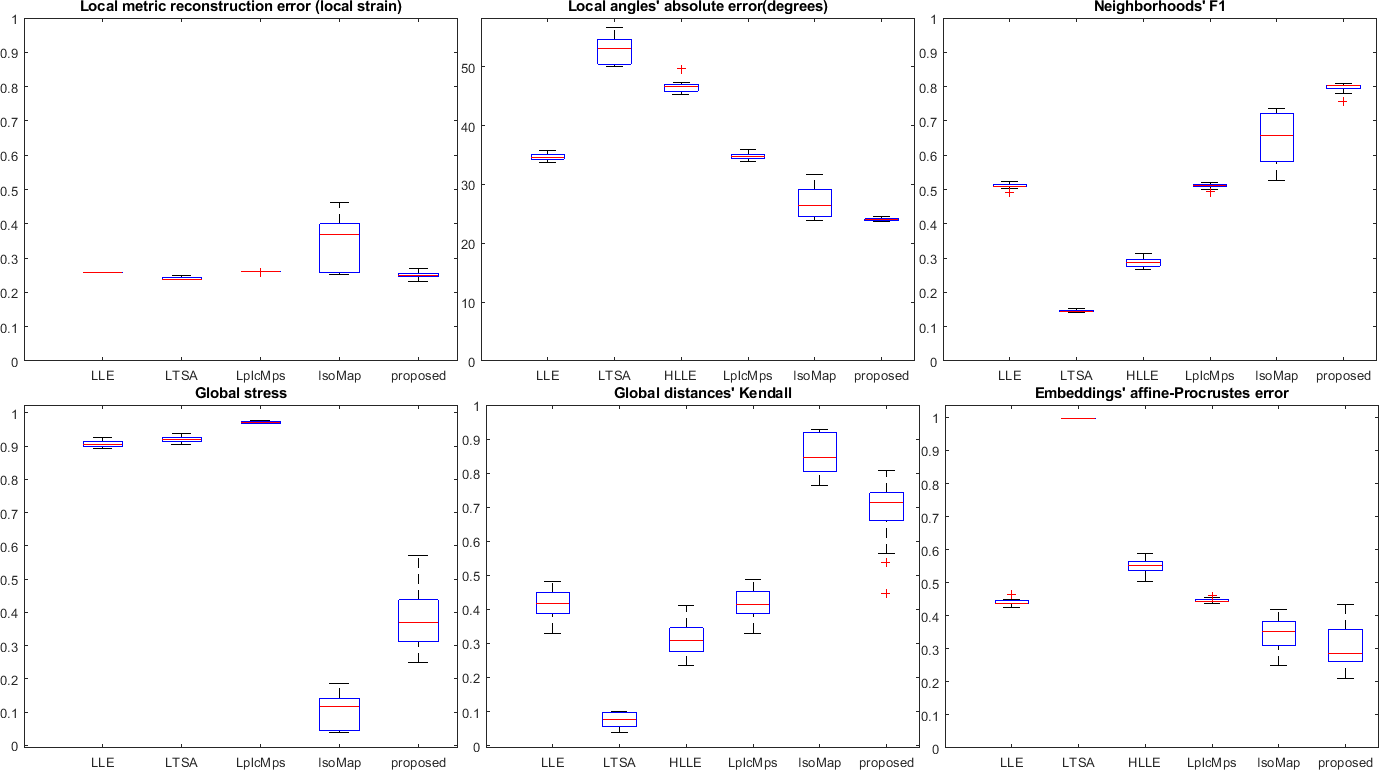}
  \caption{Recovery of the parameterization of a 5-dimensional manifold embedded in 10-dimensional space. Fifteen random samplings were performed, each of 1500 datapoints. The measures' box plots depict the maximum/minimum boundaries at whiskers, the boxes of the populations' 25th-75th percantile and the median values in red. The proposed method is the only local one, approximating IsoMap in the metrics of global consistency (Global stress that IsoMap minimizes and Kendall correlation between the Euclidean and the original geodesic distances), while retaining local consistency of neighborhoods' representation.}
\label{fig_difficult}
\end{figure}
The experimental setup starts by determining a proper distance graph, common for the proposed and the baseline methods. This is driven by the determination of each dataset's intrinsic dimension, $N$, which is estimated by the method of \citep{ref_intrdim}, on the basis of the two nearest neighbors' distances' ratio.  Then, the $N$ - nearest neighbors' graph is determined and its adjacency matrix $A$ is rendered symmetric by $A\leftarrow\max(A,A^T)$, so that it corresponds to a distance graph. 

In the cases of low intrinsic dimension($N\leq 5$), we have also tested the method's  performance as the number of neighbors varies. Namely, in each case, we have considered the least number of neighbors that allowed any of the \textit{local spectral methods} to approximate the genuine manifold's parameterization

Concerning the convergence tolerance, employed in the performed experiments, for the proposed embedding algorithm, we have left the algorithm fully converge to the single precision machine epsilon, $10^{-7}$ and to the double precision one, $10^{-16}$, for the convergence error sequential variation.
\subsection{Datasets, baseline methods and evaluation metrics}
The proposed embedding method has been tested against two requirements : a) local and global geometric fidelity and b) maintenance or enhancement of the data discrimination.

The metrics selected to evaluate a) are for the \textit{local characteristics}: \textbf{Lcl-dist}: local distances' absolute relative error, \textbf{Lcl-ang}: local frames' angles' absolute error in degrees, \textbf{Lcl-mtrc}: strain of the local metric reconstruction, \textbf{Lcl-cont} and \textbf{Lcl-trust}: neighborhoods' continuity and trustworthiness,  \textbf{Lcl-prec}, \textbf{Lcl-rec} and \textbf{Lcl-F1}: neighborhoods' precision, recall and F1 score, while for the \textit{global characteristics}: \textbf{Glbl-mtrc} stress of the reconstructed geodesic distances, \textbf{Glbl-corr}: Pearson, Spearman and Kendall correlation between the Euclidean and the geodesic distances, \textbf{Glbl-prm}: procrustes error of the affine alignment between the determined embedding and  the underlying manifold's parameterization (only for synthetic datasets). 
On the other hand, to evaluate the capability of the embedding to discriminate data clusters, for each data point of a labeled dataset, we consider its $N$ - nearest neighbors, with respect to the tested embedding. Then, the label of each neighbor is weighted by the logistic function of its Euclidean distance from the query point. The query point  obtains the label of  the maximum average weight. The data points' grouping is measured by \textbf{ACC}: its accuracy, \textbf{NMI}: its Normalized Mutual Information with  the ground truth and \textbf{ARI}: its Adjusted Rand Index.

The embedding method has been evaluated on both synthetic and real data and compared to representative embedding methods. Specifically, the dimensionality reduction toolbox of \citep{toolbox} is employed, using its implementations of local embedding methods, with the exception of UMAP implementation \citep{ref_umapmat}. Namely, additionally to UMAP and t-SNE \citep{ref_tsne}  from the class of stochastic local methods, the proposed method is compared to IsoMap \citep{ref_isomap} and Diffusion Maps (DffMps) \citep{ref_diffusion}, which are global methods, expected to perform better than the proposed one in reproducing the global manifold structure and LLE \citep{ref_lle}, LTSA \citep{ref_ltsa}, Laplacian Eigenmaps (LplcMps) \citep{ref_laplacian} and Hessian Eigenmaps (LplcMps) \citep{ref_hess}, which are local methods expected to distort global manifold structure. For fairness, the local methods' implementations have been properly modified so as to operate on graphs rather than features vectors, following the adaptation proposed in \citep{ref_lle1}.

The considered \textit{synthetic datasets} are: the swiss roll ($\mathbb{R}^{3}\rightarrow\mathbb{R}^{2}$), the "difficult" dataset of \citep{toolbox} ($\mathbb{R}^{10}\rightarrow\mathbb{R}^{5}$), the flat torus ($\mathbb{R}^{4}\rightarrow\mathbb{R}^{3}$) and the Klein's bottle ($\mathbb{R}^{4}\rightarrow\mathbb{R}^{3}$) , while the considered \textit{real datasets} are the test sets of MNIST and FMNIST images' datasets  ($\mathbb{R}^{784}\rightarrow\mathbb{R}^{13}$) and an RNA-seq data set, related to lung cancer and collected from \citep{pbmc} (1625 cases with 51 different diagnoses, $\mathbb{R}^{60660}\rightarrow\mathbb{R}^{30}$ embedding).

\subsection{Results}

\begin{figure}
  \centering
 \includegraphics[width=1\linewidth]{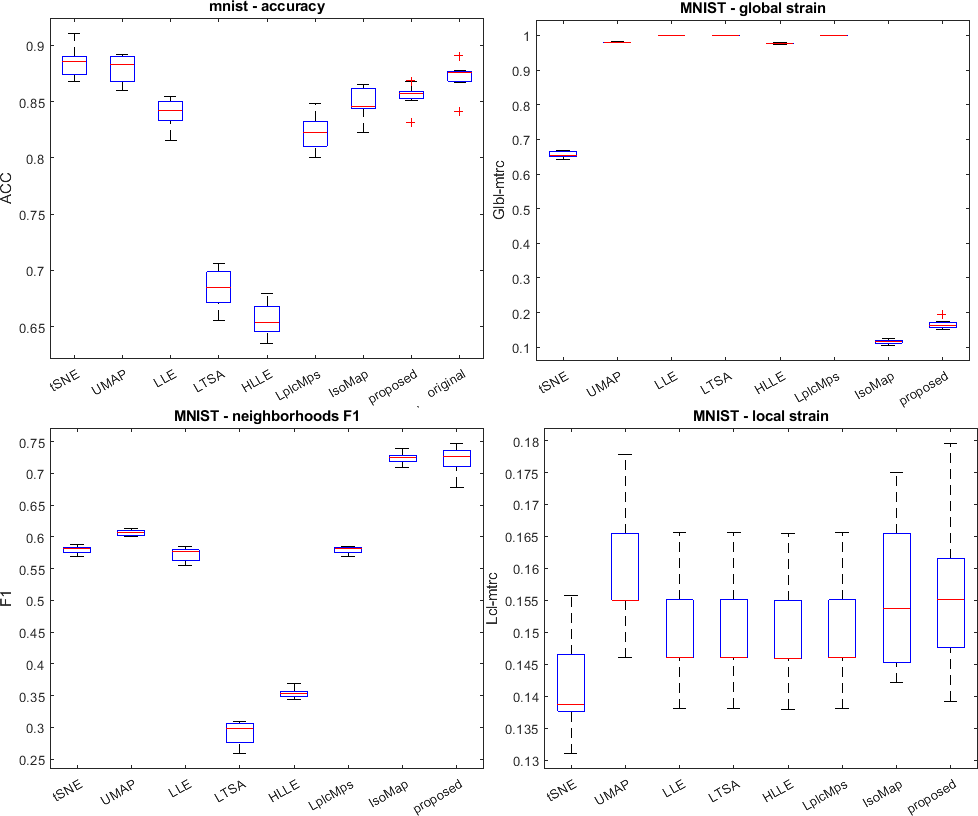}
  \caption{MNIST dataset 10-fold experiment. The measures' box plots depict the maximum/minimum boundaries at whiskers, the boxes of the populations' 25th-75th percantile and the median values in red. Possible outliers are marked by red crosses. The proposed method performs practically the same as IsoMap, either in terms of clustering or local/global geometric consistency. The stochastic methods slightly outperform their clustering accuracy. Probably outliers within neighborhoods confine this result, since also the classification according to the "original" neighborhoods of the graph is outperformed by the stochastic methods.}
\label{fig_mnist}
\end{figure}

\begin{table}
  \caption{Evaluation of the clusters' representation for the RNA-seq data set, related to lung cancer}
  \label{tab_pbmc}
  \centering
  \begin{tabular}{p{1.5cm}p{1cm}p{1cm}p{1cm}p{1cm}p{1cm}p{1cm}p{1.2cm}}
    \toprule
    &\multicolumn{7}{c}{Diagnosis classification within embeddings' neighborhoods}  \\
    \cmidrule(r){2-8}
     &LLE  &LTSA &LplcMps &IsoMap &t-SNE &UMAP &proposed  \\
    \midrule
    ACC & 0.6954 & 0.6640 & 0.6985 & 0.6825 & 0.6917 & 0.7040 & 0.7126\\
    NMI & 0.2934 & 0.2660 & 0.3006 & 0.2811 & 0.3048 & 0.3213 & 0.3298\\
    ARI   & 0.3356 & 0.2525 & 0.3525 & 0.3020 & 0.3461 & 0.3742 &0.3821\\
\midrule
&\multicolumn{7}{c}{Local/global geometric consistency}  \\
    \cmidrule(r){2-8}
    Lcl-F1 & 0.5258 & 0.3099 & 0.5410 & 0.5822 & 0.6788 & 0.7160 & 0.7698\\
    Lcl-mtrc & 0.0756 & 0.0752 & 0.0761 & 0.1087 & 3.2433 & 0.1788 & 0.0566\\
    Glbl-mtrc & 0.9428 & 0.9383 & 0.9903 & 0.1519 & 34.28 & 3.1886 & 0.1589\\
    \bottomrule
  \end{tabular}
\end{table}
\textit{Synthetic datasets' results} : The widely used case of "swiss roll" surface unfolding is considered, varying the number of neighbors and  repeating the random sampling 15 times per neighborhood size. The evaluation of the embeddings' quality is summarized in Fig. \ref{fig_swiss}. There the proposed method is compared with standard local embedding methods and the globally optimal (in terms of isometry) IsoMap, with respect to measures of local (metric, angular and neighbors') and global (Euclidean -- to -- geodesic distances' stress and correlation) consistency. The proposed is the only local method approximating the performance of IsoMap, in terms of global consistency, while retaining the local  geometric data and the neighboring relations. Local metric conservation is adversarial to the integrability of the computed embedding, due to the curvature of the underlying manifold. After all, the deformation of the local metric for the sake of integrability is the error functional minimized by the proposed method. Consistently, the proposed method results to a more conservative local metric distortion than IsoMap, however not following the other local methods' invariance with respect to the neighborhood size. Actually, while the number of neighbors increases, the intra-neighborhood distances deviate more intensely from the true geodesic ones and if a globally accurate isometric embedding is recovered, these distances should be stretched to approximate their true geodesic counterparts.

From the synthetic datasets provided in \citep{toolbox}, the "difficult" dataset, which is a 5-dimension manifold embedded in $\mathbb{R}^{10}$ 
is designed to challenge the dimensionality reduction methodologies' capability to recover the underlying parametrerization in $\mathbb{R}^5$. The embeddings' evaluation with respect to this manifold is is summarized in Fig. \ref{fig_difficult}. The presented results correspond to 15 random samplings, each of 1500 points and the graph that underlies the computed embeddings is the minimal graph that achieves an intrinsic dimension 5 within all neighborhoods. The local/global geometric consistency evaluations pratctically follow the results obtained in the swiss roll unfolding experiments; the proposed method best reserves neighborhoods' data, while approximating the global geometric consistency of IsoMap. Within the latter evaluation framework, we should refer to the performance of the proposed method in the task of recovering the original manifold's parameterization. The proposed method's embedding best fits this paramaterization (affine-Procrustes error subfigure), probably favored by the integrability constraint imposed to the embedding by the considered optimization problem.

\textit{Real datasets' results} : The most characteristic case, where embeddings' loca/global geometric consistency affects data clustering is the RNA-seq dataset (table \ref{tab_pbmc}). There, the 51 different diagnoses' classes are highly imbalanced and not fully separated by global distances. The consistency between the embeddings' neighborhoods and the original ones (F1 score)
is the key element for reasoning the embeddings' clustering capability. Local distances are given only to identify the local/global geometric consistency and its correlation with the neighborhoods' consistency.  One may observe in table \ref{tab_pbmc} that for methods retaining local and/or global distance data, plausibly, local consistency is  the key for neighborhoods' consistency. IsoMap distorts it by collapsing neighboring geodesics, thus being outperformed. t-SNE and UMAP, by design, do not preserve distances. The proposed method's embedding succeeds in best representing, local structures thus achieving best clusters' representation, while it's global geometric consistency is practically the one of IsoMap.

The MINST and FMNIST datasets' test sets were used to test the stability of the embedding methods under sparse random sampling within high dimensional spaces. Specifically, the 10k data points of these datasets were randomly split in 10 disjoint groups and the embedding methods have been evaluated on these groups. The MNIST results are summarized in Fig. \ref{fig_mnist} and the ones of FMNIST are included in appendix \ref{appD}. The FMNIST results practically follow the ones of RNA-seq dataset, indicating that the preservation of the neighborhoods' structure (F1 score) sets the baseline for the embeddings' capability to represent data clusters. This is not that clear-cut in the MNIST results. While, the proposed method and IsoMap are top-scored both in terms of neighborhood preservation (F1 score) and in terms of global distances' reconstruction (global strain), they are outperformed by the density based methods, t-SNE and UMAP, in the accuracy of data clusters' representation. In this case, the clustering capability of the original graph's neighborhoods sets the barrier for the performance of the proposed method and IsoMap that prserve and arange them according to the original graph structure. On the other hand, t-SNE and UMAP modify the neighborhoods and arrange them according to the homogeneity of the data points' density, which represents clustering better than the actual datapoints' distances.

\section{Limitations and extensions}
The fundamental restriction that the proposed approach imposes to the data embedding problem is that any sought data - embedding correspondence should be encoded in terms of a symmetric relation between datapoints that are considered neighboring according to his relation. This setting fundamentally excludes requirements in connection with relations between non-neighboring datapoints (e.g. negative sampling, Barnes–Hut approximation) and even worse, requirements about the global shape of data. Such requirements are crucial for algorithms like UMAP \citep{ref_umap} and t-SNE \citep{ref_tsne} that are highly efficient in representing data clustering. In terms of the proposed approach, the restriction to local-only relations is a cost paid in favor of the variational formulation of the embedding problem and the requirement that the method does not rely on dense pairwise data. 
Namely, incorporation of distant relations in a variational formulation would lead to covariance-like operators that are densely represented on data, thus violating the sought sparsity of the processed data. Concerning the cases that non-symmetric relations are adopted (e.g. k-nn graphs),  they can be relaxed either by forcing symmetry (e.g. for k-nn graphs) or by inserting points (e.g. in cases that bidirectional edges have different weights).

The consistency between the operators included in the variational formulation of the embedding problem and the locality of the distance graph's edges also imposes the main technical challenge to the implementation of the optimal embedding method of Sct. \ref{Sct_mthd}. Though, the iterative scheme of \eqref{eq_ELeqs} is theoretically guaranteed to converge to a globally optimal embedding, the on-graph realization of the involved differential operators is up to a second order approximation. Thus the presence of outliers in the neighborhoods of the distance graph, though it does not affect the convergence, it may result to poor operators' estimates, thus affecting the quality of the embedding.

Concerning the proposed method's computational profile, its memory demands remain linear in the number of datapoints, $K$, in accordance with all local embedding methods. However, its time complexity, per iteration, is between $O(K)$ and  ${O(K)}^{1+1/\text{dimensions}}$, while UMAP \citep{ref_umap} and LargeVis \citep{ref_largevis} iterations' time demands are $O(K)$. Shifting to an algebraic multigrid (AMG) preconditioner for the Laplacian inversion (Sct. \ref{Sct_impl_alg})  guarantees $O(K)$, however imposing challenges in the convergence of the iterative scheme, due to the non-stationarity of the AMG preconditioning, across iterations.

Finally, the proposed on-graph implementation of the optimal embedding scheme of \eqref{eq_ELeqs}, though general in terms of data diversity and  sparse in terms of data volume, it necessarily remains recursive and non-parametric. As a consequence, the embedding of new data points, requires their incorporation in the distance graph and usage of Nyström approximation to extend the embedding. To restore optimality of these extensions, we should use them to extend the alignment matrices $Q|_\infty$ and then repeat the iterative embedding scheme updating only the new points. As it becomes evident, such an inference is computationally expensive,  fetching the whole training data for each out of sample extension. Such a deficiency is common in pointwise, non-parametric estimation cases, since such estimations lack a functional form covering the whole data space. 

On the other hand, parametric embedding schemes that provide such a functional form, require vector representations for the processed data, thus lacking the representation-free feature that allows for operating heterogeneous data. Autoencoders \citep{ref_autoenc}, \citep{ref_autoencVar}, is the source point of the parametric, non-linear data embedding. The objectives optimized by the learned models focus on reconstructing the data from a mapping (encoder) that embeds them in a low-dimensional vector space. The coupling between the embedding and the reconstruction (decoder) expands the class of the optimal embeddings to the whole class of diffeomorphisms, having only regularization of the embedding to impose restrictions within the class. Thus, local (differential) or global (integral) data - embedding correspondences cannot be imposed in the Autoencoders' objective function. Shifting to deep architectures (convolutional \citep{ref_neurop} or attention-based \citep{ref_attention}) partially resolves this, if the datapoints' graph is fixed and gridded. To cover the fully intrinsic case of irregular data graphs, one should shift to the graph deep learning  framework (e.g. Graph Attention \citep{ref_gAttn}), however not avoiding the $O(\log(K))$ graph update step per inference. 

A hybrid class of parametric methods emerges by minimally adapting the local point-wise embedding methods, to learn shallow neural networks, minimizing each method's objective function. These adaptations adopt the data graph only for creating cost function, used during the learning process. The models consider  vector data representations as inputs as Autoencoders also do.

The proposed approach is not inherently bound to non-parametric estimation, since it fundamentally relies on the functional equation of the optimal embedding. The operators involved in these equations can be analytically evaluated on parametric models, thus enabling a possible reduction of the original operator-learning problem  to a mapping-learning one. This is a technically involved extension, both in terms of theory and implementation and it is the major direction of our on-going research effort.




\medskip

{
\small
\bibliographystyle{unsrtnat}
\bibliography{my_ref_paper}


\appendix

\section{Inner product of differential 1-forms and their divergence}
\label{app0}
Concerning the notation adopted for representing vectors, matrices and tensors and their contractions, Einstein summation convention is used and in the matrix form version of the expressions, the vectors are treated as columns.

An 1-form $\eta$ on an $N$-dimensional manifold $\Omega$ can always  be written as a linear combination of the manifold's local coordinates' differentials, e.g. for an $N$-dimensional manifold with $dx=[dx_i]_{i=1...N}$, $\eta$ reads $\eta=dx_i h^i=dx^T h$, where $\eta=[h_i]_{i=1...N}$ is the tangent vector of the corresponding infinitesimal transport, with respect to the selected local coordinates. 

Hodge $\star$ operator, allows for taking the space of transports orthogonal to the ones included in the operated differential form. Acting on the 1-form $\eta$, the operator returns an $(N-1)$-form, $\star\eta=\star dx_i h^i$, having $\eta$ as normal vector.  This is achieved if for each $dx_i$, $\star dx_i$ corresponds to differential transports normal to $dx_i$, which is equivalent to asking that for each $dx_j$, the exterior (wedge) product $dx_j\wedge \star dx_i$ equals the interior product of $dx_j$ and $dx_i$. If such an identification is achieved, then, using the Hodge operator, one can determine the inner product of differential forms, point-wise as $\langle \eta_1 , \eta_2 \rangle =\eta_1 \wedge \star\eta_2 = \eta_2 \wedge \star\eta_1$, which, in turn allows for identifying the squared $L^2$ norm of a differential form as $\langle \eta , \eta\rangle$

As a $(N-1)$-form, $\star dx_i$ can be expressed as a linear combination of the differential $(N-1)$-forms $[dx_{k^*}]_{k=1...N}$, where $dx_{k^*} : dx_k\wedge dx_{k^*}=dx_1\wedge...\wedge dx_N$. So, let $\star dx_i =dx_{k^*}  \gamma^k_i$ for a proper matrix of coefficients $\gamma$. Then, $dx_j\wedge \star dx_i=\gamma^j_i (dx_1\wedge...\wedge dx_N)$ should equal the interior product of $dx_j$ and $dx_i$. By the definition of the Riemann metric, let it be $g$, the interior  product of  the manifold's local coordinates' differentials reads $\langle dx_i , dx_j\rangle=(g^{-1})^j_i d\mathrm{vol}$, while their exterior (wedge) product reads $dx_1\wedge...\wedge dx_N=d\mathrm{vol} \sqrt{\det g^{-1}}$, with $d\mathrm{vol}$ denoting the volume form of the manifold. Consequently, the coefficients' matrix $\gamma$ reads $\gamma =g^{-1}\sqrt{\det g}$, thus allowing us to write
\begin{equation}
\star\eta=\sqrt{\det g} ~ h_i (g^{-1})^i_j dx_{j^*}
\end{equation}

Having defined the $\star$ operator that allows to evaluate the inner product of differential 1-forms point-wise, integration of this evaluation over the whole manifold $\Omega$ determines the inner product as a functional, in the form
\begin{equation}
\label{eqApp_innerProdDef}
(\eta_1,\eta_2)\equiv\int_{\Omega}{\langle \eta_1, \eta_2\rangle}\equiv\int_{\Omega}{\eta_1\wedge\star\eta_2}=\int_{\Omega}{{h_1}^T g^{-1}h_2 d\mathrm{vol}}
\end{equation}

In turn, this functional form of the inner product and the generalized Stokes' theorem allow for determining the dual action on the 1-forms of the differential of scalar functions (0-forms). Namely, if we consider the inner product of such a differential $df$ with an 1-form $\eta$, Stokes theorem allows us to write $(df,\eta)\equiv \int_{\Omega}{df\wedge\star\eta} =\int_{\partial \Omega}{f\star\eta} -\int_{\Omega}{f d\star\eta} =\int_{\Omega}{f (d\mathrm{\chi}_\Omega\wedge\star\eta - d\star\eta)}=(f,d\mathrm{\chi}_\Omega\wedge\star\eta - d\star\eta)$, where $\mathrm{\chi}_\Omega$ stands for the characteristic function of $\Omega$. Then the sought duality, asks for determining an operator $\mathrm{div}$, acting on differential 1-forms, such that $(df,\eta)=(f,\mathrm{div}(\eta))$. Consequently, the divergence operator $\mathrm{div}$ that satisfies this duality reads
\begin{equation}
\label{eqApp_divNeumann}
\mathrm{div}(\eta) = d\mathrm{\chi}_\Omega\wedge\star\eta - d\star\eta
\end{equation}
Since $d\mathrm{\chi}_\Omega$ results to the Dirac evaluator of $\star\eta$, \eqref{eqApp_divNeumann} determines the divergence of $\eta$ as the co-differential operator $d\star\eta$ in $\Omega$ and on $\partial\Omega$, as the inner product $\eta|_{\partial\Omega}$  of $\eta$ with the normal of $\partial\Omega$.

Then, equations of the form $\mathrm{div}(\eta) =\rho d\mathrm{vol}$ correspond to continuity equations with Neumann boundary conditions, thus allowing to simultaneously represent the equation and its boundary conditions with a single operator.

\section{Representation of differential 1-forms, their inner product and their divergence on distance graphs}
\label{appA}
Since differential 1-forms correspond to the integrands of line integrals and the graph's fundamenal line elements are its edges,  the intrinsic basis of differential 1-forms should be associated with the lines underlying graph edges. Correspondingly, the vector fields to be integrated should be realized on-graph as edge functions. As a consequence,  a generic differential 1-form $\zeta$ can be expressed in the neighborhood of a vertex $v_i$ of a graph as
\begin{equation}
\label{eqApp_diffForms}
\zeta^{v_i}=\sum_{v_j\in\Gamma(v_i)} Z^{v_i}|_{v_j}e^{v_i}|_{v_j}
\end{equation}
where $e^{v_i}|_{v_j}$ is the 1-form element along the edge $v_i\rightarrow v_j$,  $Z^{v_i}|_{v_j}$ the corresponding component of the $Z^{v_i}$ vector field and $\Gamma(v_i)$ the neighborhood of $v_i$.

In order to compare differential 1-forms, it is necessary to evaluate their inner product, which also allows for defining a quadratic error functional that measures this comparison. For any two differential 1-forms, $\zeta,\eta$, their point-wise inner product in $\Gamma(v_i)$ is determined by the inner product of the 1-form elements 
\begin{equation}
\langle e^{v_i}|_{v_j},e^{v_i}|_{v_k}\rangle=g^{v_i}|_{v_j,v_k} d\mathrm{vol}
\end{equation}
where $d\mathrm{vol}$ stands for the infinitesimal intrinsic volume and $g^{v_i}|_{v_j,v_k}$ is the Gram matrix of the tangents of the transports along the edges that source from $v_i$.

Considering the $v_i$ - centered $\mathbb{R}^{N}$ coordinates of \eqref{eq_coords}, the 1-forms of their differential, $\alpha^{v_i}(p)$, have inner products, whose corresponding Gram matrix is trivial along the geodesic paths that connect the vertices of $\Gamma(v_i)$. For any other $p$, within the convex hull of $\Gamma(v_i)$, if $v_*(p)$ is the geodesic projection of $p$ to its closest edge, the Gram matrix of  $\alpha^{v_i}(p)$ differs from the (trivial) one of $\alpha^{v_i}(v_*(p))$ by 
\begin{equation}
\langle \alpha^{v_i}(p),(\alpha^{v_i}(p))^T\rangle=I_{N\times N} d\mathrm{vol}(v_*(p))+\langle \alpha^{v_i}(v_*(p)),{\epsilon(p)}^T\rangle + \langle \epsilon(p),(\alpha^{v_i}(v_*(p)))^T\rangle +\langle \epsilon(p),{\epsilon(p)}^T\rangle
\end{equation}
where $\epsilon(p)=\alpha^{v_i}(p)-\alpha^{v_i}(v_*(p))$. Since, $v_*(p)$ is the geodesic projection of $p$, the transports of $\epsilon(p)$ are orthogonal to the ones of $\alpha^{v_i}(v_*(p))$ at $v_*(p)$ and thus, $\langle \alpha^{v_i}(v_*(p)),{\epsilon(p)}^T\rangle$ is a second order term, varying from zero only due to extrinsic rotation of the transports of $\alpha^{v_i}(v_*(p))$, which is intrinsically measured by the curvature of the underlying manifold. Consequently,
\begin{equation}
\label{eqApp_trvlz}
\langle \alpha^{v_i}(p),(\alpha^{v_i}(p))^T\rangle=(I_{N\times N}+O(|x^{v_i}(p)|^2)) d\mathrm{vol}(v_*(p))=(I_{N\times N}+O(\|G^{v_i}\|_2)) d\mathrm{vol}(v_*(p))
\end{equation}
where $G^{v_i}$ is the Gram matrix of $\Gamma(v_i)$, computed as described in Sct. \ref{Sct_mds}.

By substituting \eqref{eq_alpha} to \eqref{eqApp_trvlz}, one obtains that, up to a 2nd order approximation, the Gram matrix $\Lambda^{v_i} : \langle d\lambda^{v_i}|_{v_j} , d\lambda^{v_i}|_{v_k}\rangle =\Lambda^{v_i}|_{v_j,v_k} d\mathrm{vol}$ should satisfy
\begin{equation}
\label{eqApp_Lambda}
 (E^{v_i})^T\Lambda^{v_i}E^{v_i} = I_{N\times N}
\end{equation}
namely being the rank - $N$ pseudo-inverse of $G^{v_i}$.

With this on - graph realization of $\langle d\lambda^{v_i}|_{v_j} , d\lambda^{v_i}|_{v_k}\rangle$ in hand and letting the on-edge 1-form elements $e^{v_i}|_{v_j}$ of  \eqref{eqApp_diffForms} be identified with $d\lambda^{v_i}|_{v_j}$, the inner product of two arbitrary differential 1-forms $\zeta,\eta$ is the integral of their point-wise product over the whole manifold, $\Omega$, thus reading 
\begin{equation}
\label{eqApp_innrProd}
(\eta,\zeta)\equiv \int_{\Omega}{\langle\eta , \zeta\rangle}=\sum_{v_i}{\sum_{v_j , v_k\in\Gamma(v_i)}{\left.\Lambda^{v_i}\right|_{v_j,v_k}\int_{|\Gamma(v_i)|}{\left.H^{v_i}\right|_{v_j} \left.Z^{v_i}\right|_{v_k} d\mathrm{vol}}}}
\end{equation}
where $|\Gamma(v_i)|$ denotes the geometric realization of the $N$-dimensional convex hull of $\Gamma(v_i)$. As for the edge function $H$, it is for the representation \eqref{eqApp_diffForms} of $\eta$ the analogue of the edge function $Z$ of $\zeta$. 

This realization of the inner product allows for defining a quadratic error measuring the $N$ discrepancies $\varepsilon$ of \eqref{eq_error} on the whole manifold $\Omega$ as 
\begin{equation}
\label{eqApp_L2error}
\|\varepsilon\|^2 \equiv \sum_{i=1}^N{(\varepsilon_i , \varepsilon_i)} 
\end{equation}
Then, letting the $N$ edge functions $\mathcal{E}^{v_i}|_{v_j}\in\mathbb{R}^N$, induced by the representation $\varepsilon^{v_i}=\sum_{v_j\in\Gamma(v_i)} \mathcal{E}^{v_i}|_{v_j}d\lambda^{v_i}|_{v_j}$, enter in \eqref{eqApp_innrProd} , in place of $H$ and $Z$, one gets \eqref{eq_L2error} as the on-graph evaluation of \eqref{eqApp_L2error}. We should clarify here that constancy of $\mathcal{E}^{v_i}|_{v_j}$ within $\Gamma(v_i)$ holds identically for the coefficients $E^{v_i}|_{v_j}$ of $\alpha^{v_i}$ and up to a second order approximation for the coefficients $(\varphi(v_j)-\varphi(v_i))$ of $d\varphi^{v_i}$, a fact that will become evident later on by determining the on-graph realization \eqref{eqApp_diffFunc} of the scalar functions' differential. 

Namely, following \eqref{eqApp_diffForms}, the differential $df$ of a scalar function $f$ may be written as
\begin{equation}
\label{eqApp_diffFunc0}
df=\sum_{v_j\in\Gamma(v_i)} \partial_{v_j}f^{v_i}d\lambda^{v_i}|_{v_j}
\end{equation}
Then, at each point $p\in|\Gamma(v_i)|$, we can consider its $v_i$ - centered $\mathbb{R}^{N}$ representation of \eqref{eq_coords} and evaluate $df(p)$ as $df(p)=\sum_{v_j\in\Gamma(v_i)} \left.\nabla_x f\right|_p \cdot E^{v_i}|_{v_j}d\lambda^{v_i}|_{v_j}$. Consequently, $\partial_{v_j}f(p) =  \left.\nabla_x f\right|_p \cdot E^{v_i}|_{v_j}=f(v_j)-f(v_i)+O(\|G^{v_i}\|_2)$. By substituting $\partial_{v_j}f(p)$ in \eqref{eqApp_diffFunc0}, we get 
\begin{equation}
\label{eqApp_diffFunc}
df(p)=\sum_{v_j\in\Gamma(v_i)} (f(v_j)-f(v_i))d\lambda^{v_i}|_{v_j}(p) + o(\|G^{v_i}\|_2)
\end{equation}
It should be clarified here that the remainder $\sum_{v_j\in\Gamma(v_i)}{O(\|G^{v_i}\|_2)d\lambda^{v_i}|_{v_j}(p)}$ is $o(\|G^{v_i}\|_2)$ because it vanishes if $O(\|G^{v_i}\|_2)$ is constant in $\Gamma(v_i)$, due to the barycentric constraint $\sum_{v_j\in\Gamma(v_i)}{\lambda^{v_i}|_{v_j}(p)}=1$.  Consequently \eqref{eqApp_diffFunc} verifies the second order approximation of scalar functions' differential given in \eqref{eq_diff} in the main text.

Having in avail on-graph realizations for both the inner product of 1-forms and the scalar functions' differential,  the divergence of 1-forms can be evaluated as the inner product - dual of the differential, namely letting one of the 1-forms $\eta$ or $\zeta$ of \eqref{eqApp_innrProd}to be exact, so that inner product offers the divergence of the other. So, let  $\eta=df$, where $f$ is an arbitrary scalar function. Then, the divergence $\mathrm{div}(\zeta)$ of $\zeta$ results from the duality $(df,\zeta)=(f,\mathrm{div}(\zeta))$, which indicates that one should use \eqref{eqApp_innrProd} in order to transfer the realization of the vertex functions’ differential operator to an operator acting on edge functions.
Using this on-graph realization of scalar functions’ differential, its inner product with an 1-form $\zeta$, by \eqref{eqApp_innrProd}, reads up to a second order approximation
\begin{equation}
\label{eqApp_innrProd0}
(df,\zeta)=\sum_{v_i}{\sum_{v_j , v_k\in\Gamma(v_i)}{(f(v_j)-f(v_i))\left.\Lambda^{v_i}\right|_{v_j,v_k} \hat{Z}^{v_i}|_{v_k} \mathrm{vol}(\Gamma(v_i))}}
\end{equation}
where $\hat{Z}^{v_i}|_{v_k}$ stands for the mean value of ${Z}^{v_i}|_{v_k}$  in $|\Gamma(v_i)|$. By rearranging the summands so that one vertex of $f$ appears in each summation term, \eqref{eqApp_innrProd0} is rewritten as
\begin{multline}
\label{eqApp_innrProd1}
(df,\zeta)=-\sum_{v_i}f(v_i)\left(\sum_{v_j , v_k\in\Gamma(v_i)}{\left.\Lambda^{v_i}\right|_{v_j,v_k} \hat{Z}^{v_i}|_{v_k}\mathrm{vol}(\Gamma(v_i))}\right.\\
-\left.\sum_{v_j : v_i\in\Gamma(v_j)}\sum_{v_k\in\Gamma(v_j)}{\left.\Lambda^{v_j}\right|_{v_i,v_k} \hat{Z}^{v_j}|_{v_k}\mathrm{vol}(\Gamma(v_j))}\right)
\end{multline}
Since, on the right-hand side of this identity, no operation is applied to $f$, other than the inner product, the vertex function that multiplies $f$ is the divergence of $\zeta$, thus reading
\begin{equation}
\label{eqApp_div0}
\mathrm{div}(\zeta)|_{v_i}=\sum_{v_j : v_i\in\Gamma(v_j)}\sum_{v_k\in\Gamma(v_j)}{\mathrm{vol}(\Gamma(v_j))\left.\Lambda^{v_j}\right|_{v_i,v_k} \hat{Z}^{v_j}|_{v_k}} - \sum_{v_j , v_k\in\Gamma(v_i)}{\mathrm{vol}(\Gamma(v_i))\left.\Lambda^{v_i}\right|_{v_j,v_k} \hat{Z}^{v_i}|_{v_k}}
\end{equation}

Having this second order realization of the divergence operator in hand, we can determine the corresponding realization of the Laplacian operator, by letting $\zeta = df$, for a certain scalar function $f$.  As indicated by \eqref{eqApp_diffFunc}, up to a second order approximation this is equivalent to substituting $\hat{Z}^{v_j}|_{v_k}=f(v_k)-f(v_k)$ thus evaluating
\begin{multline}
\mathrm{div}(df)|_{v_i}=\sum_{v_j : v_i\in\Gamma(v_j)}\sum_{v_k\in\Gamma(v_j)}{\mathrm{vol}(\Gamma(v_j))\left.\Lambda^{v_j}\right|_{v_i,v_k} (f(v_k)-f(v_j))}\\
- \sum_{v_j , v_k\in\Gamma(v_i)}{\mathrm{vol}(\Gamma(v_i))\left.\Lambda^{v_i}\right|_{v_j,v_k} (f(v_k)-f(v_i))}
\end{multline}
By writing this linear combination of the values of $f$ in matrix form, we can evaluate the Laplacian operator as a sparse matrix $\mathcal{L}$ with entries $\mathcal{L}^{v_i}|_{v_k}$, for $v_k \in \Gamma^2(v_i)$ given by
\begin{multline}
\label{eqApp_lplc0}
\mathcal{L}^{v_i}|_{v_k}=\sum_{v_j : v_i\in\Gamma(v_j)}{\mathrm{vol}(\Gamma(v_j))\left.\Lambda^{v_j}\right|_{v_i,v_k}} + \delta^{v_i}|_{v_k}\sum_{v_j , v_m\in\Gamma(v_i)}{\mathrm{vol}(\Gamma(v_i))\left.\Lambda^{v_i}\right|_{v_j,v_m}} \\
-\sum_{v_j\in\Gamma(v_k)}{\mathrm{vol}(\Gamma(v_k))\left.\Lambda^{v_k}\right|_{v_i,v_j}} - \sum_{v_j \in\Gamma(v_i)}{\mathrm{vol}(\Gamma(v_i))\left.\Lambda^{v_i}\right|_{v_j,v_k}}
\end{multline}
where $\delta^{v_i}|_{v_k}$ stands for the sparse identity matrix, having 1 if $v_i=v_k$ and 0 otherwise.

In order to derive vectorized versions of the defining realizations \eqref{eqApp_innrProd} for the inner product, \eqref{eqApp_div0} for the divergence and \eqref{eqApp_lplc0} for the Laplacian, we should overcome the problem that $\Lambda$, although sparse, it has entries ordered with respect to three dimensions. However, $\Lambda^{v_i}$, as a pseudo-inverse of $E^{v_i}(E^{v_i})^T$ (refer to \eqref{eqApp_Lambda}), may be expressed as a vector product of $N$ edge functions, namely the columns of the pseudo-inverses of the local frames $E^{v_i}$ of \eqref{eq_LclFrms}. Thus, we can express inner product's weighting factors as
\begin{align}
\left.\Lambda^{v_i}\right|_{v_j,v_k} \mathrm{vol}(\Gamma(v_i)) &=\left\langle{E_{-}}^{v_i}|_{v_j}~,~{E_{-}}^{v_i}|_{v_k}\right\rangle_{\mathbb{R}^N}
\label{eqApp_LambdaImpl}
\\
{E_{-}}^{v_i}|_{v_j} &\equiv \sqrt{\mathrm{vol}(\Gamma(v_i))} V^{v_i}|_{v_j}(S^{v_i})^{-1/2}\mathbf{1}_N
\label{eqApp_LclFrmsInv}
\end{align}
where, again, $\mathbf{1}_N$ denotes a (\# of neighbors) - by - $N$ identity matrix, selecting the largest $N$ eigenvalues of $S^{v_i}$ (please refer to \eqref{eq_LclFrms}), thus rendering $V^{v_i}(S^{v_i})^{-1/2}\mathbf{1}_N$ the rank-$N$ pseudo-inverse of $E^{v_i}$.

By substituting \eqref{eqApp_LambdaImpl} in \eqref{eqApp_innrProd}, the inner product may itself be implemented  as a vector product of $N$ edge functions
\begin{equation}
\label{eqApp_innrProdMat}
\left.\langle\eta,\zeta\rangle\right|_{v_i}=\left\langle \sum_{v_j\in\Gamma(v_i)}{\hat{H}^{v_i}|_{v_j}{E_{-}}^{v_i}|_{v_j}}~~, \sum_{v_k\in\Gamma(v_i)}{\hat{Z}^{v_i}|_{v_k}{E_{-}}^{v_i}|_{v_k}}\right\rangle_{\mathbb{R}^N}
\end{equation}
where the edge functions $\hat{H},\hat{Z}$ are the $H,Z$ coefficients of $\zeta,\eta$ evaluated at the mean value theorem - verifying points $p^{v_i}_*\in\Gamma(v_i) : \int_{|\Gamma(v_i)|}{\left.H^{v_i}\right|_{v_j} \left.Z^{v_i}\right|_{v_k} d\mathrm{vol}}=H^{v_i}|_{v_j}(p^{v_i}_*)Z^{v_i}|_{v_k}(p^{v_i}_*) \mathrm{vol}(\Gamma(v_i))$.

In turn, substitution of \eqref{eqApp_LambdaImpl} in the realization \eqref{eqApp_div0} of the divergence, the operator takes the form of a vector product of $N$ edge functions, thus reading
\begin{multline}
\label{eqApp_div}
\left.\mathrm{div}(\zeta)\right|_{v_i}=\sum_{v_j : v_i\in\Gamma(v_j)}{\left\langle{E_{-}}^{v_j}|_{v_i} ~~, \sum_{v_k\in\Gamma(v_j)}{{E_{-}}^{v_j}|_{v_k}\left.\hat{Z}^{v_j}\right|_{v_k}}\right\rangle_{\mathbb{R}^N}}\\
-\left\langle \sum_{v_j\in\Gamma(v_i)}{E_{-}}^{v_i}|_{v_j}~~ , \sum_{v_k\in\Gamma(v_i)}{{E_{-}}^{v_i}|_{v_k} \left.\hat{Z}^{v_i}\right|_{v_k} }\right\rangle_{\mathbb{R}^N}
\end{multline}

Finally, the same substitution in the realization \eqref{eqApp_lplc0} of the Laplacian operator, offer its vector product evaluation
\begin{multline}
\label{eqApp_lplc}
\mathcal{L}^{v_i}|_{v_k}=\sum_{v_j : v_i\in\Gamma(v_j)}{\left\langle{{E_{-}}^{v_j}|_{v_i} ~~,~~{E_{-}}^{v_j}|_{v_k}}\right\rangle_{\mathbb{R}^N}}+  \delta^{v_i}|_{v_k} \left\langle \sum_{v_j\in\Gamma(v_i)}{E_{-}}^{v_i}|_{v_j}~~ , \sum_{v_m\in\Gamma(v_i)}{{E_{-}}^{v_i}|_{v_m}}\right\rangle_{\mathbb{R}^N} \\
-\left\langle{E_{-}}^{v_k}|_{v_i} ~~, \sum_{v_j\in\Gamma(v_k)}{{E_{-}}^{v_k}|_{v_j}}\right\rangle_{\mathbb{R}^N} - \left\langle{E_{-}}^{v_i}|_{v_k} ~~, \sum_{v_j\in\Gamma(v_i)}{{E_{-}}^{v_i}|_{v_j}}\right\rangle_{\mathbb{R}^N}
\end{multline}

\section{The Euler - Lagrange equations of the optimal Euclidean embedding}
\label{appB}
Concerning the notation adopted for representing vectors, matrices and tensors and their contractions, Einstein summation convention is used and in the matrix form version of the expressions, the vectors are treated as columns.

Since the sought embedding in $\mathbb{R}^N$, $\varphi=[\varphi_j]_{j=1..N}$ should minimize the discrepancies of \eqref{eq_error}, $\varepsilon_j=d\varphi_j-\alpha_i Q^i_j$, the squared $L^2$ norm of these 1-forms is considered to measure their magnitude. Using the evaluation \eqref{eqApp_innerProdDef} of the inner product of differential 1-forms, the corresponding error functional reads
\begin{equation}
\label{eqApp_errFunc}
\mathcal{J}[\varphi,Q]=\frac{1}{2}\|(\varepsilon,\varepsilon)\|^2=\frac{1}{2}\int_{\Omega}\langle\varepsilon_j , \varepsilon^j \rangle=\frac{1}{2}\int_{\Omega}\varepsilon_j \wedge \star\varepsilon^j
\end{equation}
where $\Omega$ is the space to be embedded, $\wedge$ stands for the wedge (exterior) product and $\star$ stands for the Hodge duality operator.

By varying $\varphi,Q$ around the optimal pair $\tilde{\varphi},\tilde{Q}$ as $\varphi=\tilde{\varphi}+\epsilon_\varphi \delta\varphi$ and $Q=\tilde{Q}+\epsilon_Q\delta Q$ and due to the symmetry of $\langle,\rangle$, the error functional reads $\mathcal{J}[\varphi,Q]=\int_{\Omega}{(\epsilon_\varphi \delta\varphi_j-\epsilon_Q\alpha_i\delta Q^i_j)\wedge\star\varepsilon^j|_{\tilde{\varphi},\tilde{Q}}}+O({\epsilon_\varphi}^2)+O({\epsilon_Q}^2)$. As a consequence, the $Q$ - variation of $\mathcal{J}[\varphi,Q]$ read
\begin{equation}
\label{eqApp_deltQ}
\left.\frac{d}{d\epsilon_Q}\mathcal{J}[\varphi,Q]\right|_{\epsilon_Q=0} =-\int_{\Omega}{\alpha_i\delta Q^i_j\wedge\star\varepsilon^j} 
\end{equation}
while the $\varphi$ - variations read  $\left.\frac{d}{d\epsilon_\varphi}\mathcal{J}[\varphi,Q]\right|_{\epsilon_\varphi=0} =\int_{\Omega}{d\delta \varphi_j\wedge\star\varepsilon^j} $. Then by the generalized Stokes theorem, applied to $d(\delta \varphi_j\wedge\star\varepsilon^j)=d\delta \varphi_j\wedge\star\varepsilon^j+\delta \varphi_j d\wedge\star\varepsilon^j$, the $\varphi$ - variations of $\mathcal{J}[\varphi,Q]$ read
\begin{equation}
\label{eqApp_deltPhi}
\left.\frac{d}{d\epsilon_\varphi}\mathcal{J}[\varphi,Q]\right|_{\epsilon_\varphi=0} =\int_{\partial\Omega}{\star\varepsilon^j \delta \varphi_j} -\int_{\Omega}{\delta \varphi_j d\star\varepsilon^j} = \int_{\Omega}\delta\varphi_j \mathrm{div}(\varepsilon^j)
\end{equation}
Here the divergence, $\mathrm{div}$, of 1-forms stands for the inner product - dual of the funcions' differential, as determined in \eqref{eqApp_divNeumann} to combine both the co-differential operator $d\star$ in $\Omega$ and the inner product with the normals of $\partial\Omega$,  on $\partial\Omega$.

Considering the  $Q$ - stationarity of $\mathcal{J}[\tilde{\varphi},\tilde{Q}]$ first, we should point out that $\delta Q$ is not free, since it should hold $Q^T Q=\tilde{Q}^T\tilde{Q}=I_{N\times N} \Leftrightarrow \delta Q^T \tilde{Q}+\tilde{Q}^T\delta{Q}=\epsilon_Q \delta Q^T \delta Q$. Consequently, at $\epsilon_Q=0$, $\delta Q^T \tilde{Q}+\tilde{Q}^T\delta{Q}=0$. Since $\tilde{Q}^T \delta Q$ should be skew-symmetric, the freely varying components are the lower (or upper) triangular components of $\tilde{Q}^T \delta Q$. Letting $\delta\theta_j^i$ denote these components, for $i>j$, and since $\tilde{Q}\tilde{Q}^T=I_{N\times N}$ the integrand of \eqref{eqApp_deltQ} reads $\alpha_i\delta Q^i_j\wedge\star\varepsilon^j=\delta\theta^k_j(\alpha_i\tilde{Q}^i_k\wedge\star\varepsilon^j -\alpha_i\tilde{Q}^i_j\wedge\star\varepsilon^k )$, with the summation convention applied for $k>j$. Due to the arbitrariness of $\delta\theta_j^k$, the  $Q$ - stationarity of $\mathcal{J}[\tilde{\varphi},\tilde{Q}]$ is equivalent to the $(k,j)$-symmetry requirement $\alpha_i\tilde{Q}^i_k\wedge\star\varepsilon^j =\alpha_i\tilde{Q}^i_j\wedge\star\varepsilon^k$. Given that the term $\alpha_i \tilde{Q}_k^i\wedge \star \alpha_m \tilde{Q}_j^m=\langle \alpha_i \tilde{Q}_k^i , \alpha_m \tilde{Q}_j^m \rangle$ of $\alpha_i \tilde{Q}_k^i\wedge \star\varepsilon^j$ is $(k,j)$-symmetric, due to the symmetry of $\langle,\rangle$, the $Q$ - stationarity of $\mathcal{J}[\tilde{\varphi},\tilde{Q}$ reduces to 
\begin{equation}
\langle \alpha_i \tilde{Q}_k^i , d\tilde{\varphi}^j \rangle = \langle \alpha_i \tilde{Q}_j^i , d\tilde{\varphi}^k \rangle
\end{equation}
This $(k,j)$-symmetry requirement may be expressed as a matrix symmetry reqirement in the form
\begin{equation}
\label{eqApp_ELq} 
B_{\tilde{\varphi}}\tilde{Q}=\tilde{Q}^T{B_{\tilde{\varphi}}}^T ~~ , ~~ \text{where} ~~ B_{\tilde{\varphi}} : \langle \alpha_i , d\tilde{\varphi}^j \rangle = (B_{\tilde{\varphi}})^j_i d\mathrm{vol}
\end{equation}
If the singular matrix decomposition of $B_{\tilde{\varphi}}$ reads $L_{\tilde{\varphi}} \Sigma_{\tilde{\varphi}} {R_{\tilde{\varphi}}}^T$, then $\tilde{Q}=R_{\tilde{\varphi}}\sigma {L_{\tilde{\varphi}}}^T$, where $\sigma$ is a diagonal matrix with either 1 or -1. By substituting this $\tilde{Q}$ back in $\mathcal{J}[\tilde{\varphi},\tilde{Q}]$, one gets 
\begin{equation}
\mathcal{J}[\tilde{\varphi},\tilde{Q}]=\frac{1}{2} \int_{\Omega}{\langle d\tilde{\varphi}_j , d\tilde{\varphi}^j \rangle} -\int_{\Omega}{\mathrm{tr} (\Sigma_{\tilde{\varphi}} \sigma) d\mathrm{vol}} + \frac{1}{2} \int_{\Omega}{\langle \alpha_j , \alpha^j \rangle}
\end{equation}
and since singular values are non-negative, $\sigma$ should be $I_{N\times N}$, in order to minimize $\mathcal{J}[\tilde{\varphi},\tilde{Q}]$, thus offering  \eqref{eq_ELq}  for determining $\tilde{Q}$.

The next step is to consider the  $\tilde{\varphi}$ - stationarity of $\mathcal{J}[\tilde{\varphi},\tilde{Q}]$, which zeros the right-hand side of \eqref{eqApp_deltPhi}. Then, arbitrariness of $\delta\varphi$, leads to the requirement $\mathrm{div}(\varepsilon^j )=0$, thus resulting to the Euler-Lagrange equations \eqref{eq_ELphi} for $\tilde{\varphi}$
\begin{equation}
\label{eqAp_ELphi} 
\mathrm{div}(d\tilde{\varphi}_j) =\mathrm{div}(\alpha_i \tilde{Q}^i_j) \Leftrightarrow 
\begin{cases}
d(\star d\tilde{\varphi}_j) = d\left(\star\alpha_i \tilde{Q}^i_j\right) \\ 
 \left.(d\tilde{\varphi}_j)\right|_{\partial\Omega} = \left.\left(\alpha_i \tilde{Q}^i_j\right)\right|_{\partial\Omega}
\end{cases} 
\end{equation}
where by $\left.(\cdot)\right|_{\partial\Omega}$  we denote the inner product of an 1-form on $\partial\Omega$ with the normals of $\partial\Omega$.

\section{Proof of theorem \ref{th_converge}}
\label{appC}
Before entering the convergence analysis of the iterative scheme of \eqref{eq_ELiter}, we should determine the homogenous space $[\tilde{\varphi},\tilde{Q}]$ of the solutions of the Euler-Lagrange equations \eqref{eqApp_ELq} and \eqref{eqAp_ELphi}.
\begin{lemma}
\label{lem_hSpace}
The solutions $\tilde{\varphi},\tilde{Q}$ of \eqref{eqApp_ELq} and \eqref{eqAp_ELphi} are unique up to simultaneous rigid orthogonal transformations and constant translations of $\tilde{\varphi}$
\end{lemma}
\begin{proof}
Starting from \eqref{eqApp_ELq}, two pairs $\tilde{\varphi}|_1,\tilde{Q}|_1$ and $\tilde{\varphi}|_2,\tilde{Q}|_2$  that satisfy it should be related, via an orthogonal matrix $C$, as $\tilde{Q}|_2=\tilde{Q}|_1 C$, $d\tilde{\varphi}|_2=d\tilde{\varphi}|_1 C$. Thus, $C$, except of being orthogonal, should correspond to a Jacobian matrix, which implies that $d\tilde{\varphi}|_1 \wedge dC=0$. Exploiting the orthogonality of $C$ one may write this integrability condition as $d\tilde{\varphi}|_2 \wedge C^T dC=0$. Then, expressing the $N\times N$ 1-forms $C^T dC$  in terms of $d\tilde{\varphi}|_2$, we may write $(C^T dC)_{i,j} =(d\tilde{\varphi}|_2)_k \Theta^k_{i,j}$. Since $C$ is orthogonal, $C^T dC$ is skew symmetric, thus implying that $\Theta^k_{i,j}=-\Theta^k_{j,i}$. On the other hand, $d\tilde{\varphi}|_2 \wedge C^T dC=0$ is equivalent to the symmetry condition  $\Theta^k_{i,j}=\Theta^i_{k,j}$. By alternating symmetry and askew-symmetry we have  $\Theta^k_{i,j}=\Theta^i_{k,j}=-\Theta^i_{j,k}=-\Theta^j_{i,k}=\Theta^j_{k,i}=\Theta^k_{j,i}=-\Theta^k_{i,j}$, thus implying that $\Theta^k_{i,j}=0\Leftrightarrow dC=0$. Since $dC=0$, if any of $(\tilde{\varphi}|_1,\tilde{Q}|_1)$ or $(\tilde{\varphi}|_2,\tilde{Q}|_2)$ satisfies \eqref{eqAp_ELphi} then the equation is also satisfied by the other pair.

So, $d\tilde{\varphi}$ and $\tilde{Q}$ that resolve \eqref{eqApp_ELq} and \eqref{eqAp_ELphi} differ only by simultaneous rigid orthogonal transformations. Within this equivalence class, the identification $d\tilde{\varphi|_1}=d\tilde{\varphi|_2}$ is equivalent to $\tilde{\varphi|_1}=\tilde{\varphi|_2} + \text{const}$. 
\end{proof}

Having determined the equivalence class of the solutions of the Euler-Lagrange equations, we consider the scheme \eqref{eq_ELiter} of alternating resolutions of \eqref{eqAp_ELphi} and \eqref{eqApp_ELq} to test if it converges within this equivalence class
\begin{subequations}
\begin{align}
\left.\tilde{\varphi}\right|_{n+1} &: \mathrm{div}(d\left.\tilde{\varphi}\right|_{n+1}) = \mathrm{div}\left(\alpha \tilde{Q}|_n\right)
\label{eqApp_phiIter}
\\
\tilde{Q}|_{n+1} &: B_{\tilde{\varphi}|_{n+1}}\tilde{Q}|_{n+1}={\tilde{Q}|_{n+1}}^T {B_{\tilde{\varphi}|_{n+1}}}^T
\label{eqApp_qIter}
\end{align}
\end{subequations}
Since $\mathcal{J}[\varphi,Q]$ is the squared $L^2$ norm of the approximation discrepancies $\varepsilon[\varphi,Q]$ of \eqref{eq_error}, based on inner product, $\mathcal{J}[\varphi,Q]$ is strictly convex with respect to $\varepsilon$ and consequently, with respect to $d\varphi$ and convex with respect to $Q$, since its variation $\delta Q$ has $\frac{N(N-1)}{2}$ independent components, instead of $N$. Consequently, if $\mathcal{J}[\tilde{\varphi}|_{n},\tilde{Q}|_{n}]$ is decreasing, it converges to its global minimum and the corresponding minimizers $\tilde{\varphi}|_\infty,\tilde{Q}|_\infty$ may vary only due to $\tilde{Q}|_\infty$ and due to constant shifts of $\tilde{\varphi}|_\infty$.

In appendix \ref{appB} , it is shown that given $\tilde{\varphi}|_{n}$, ${\tilde{Q}|_{n}}$ as results from \eqref{eqApp_qIter}, minimizes $\mathcal{J}[\tilde{\varphi}|_{n},Q]$ and consequently  $\mathcal{J}[\tilde{\varphi}|_{n},\tilde{Q}|_{n}]\leq\mathcal{J}[\tilde{\varphi}|_{n},\tilde{Q}|_{n-1}]$.  So, in order to verify the sought convergence, one should test if  $\mathcal{J}[\tilde{\varphi}|_{n+1},\tilde{Q}|_{n}]\leq\mathcal{J}[\tilde{\varphi}|_{n},\tilde{Q}|_{n}]$.

Considering the Hodge decomposition of the 1-forms $\alpha\tilde{Q}|_n$, the defining equation of $\tilde{\varphi}|_{n+1}$ in \eqref{eqApp_phiIter} is the one that determines the exact component of the decomposition. As a result, the discrepancies $\varepsilon[\tilde{\varphi}|_{n+1},\tilde{Q}|_{n}]=d\tilde{\varphi}|_{n+1}-\alpha\tilde{Q}|_n$ consist of the remaining terms of the Hodge decomposition (co-exact and harmonic), thus satisfying 
\begin{equation}
\label{eqApp_coclosed}
d(\star\varepsilon[\tilde{\varphi}|_{n+1},\tilde{Q}|_{n}])=0
\end{equation}
Now, if, instead of $\tilde{\varphi}|_{n+1}$, one employs functions $\varphi=\tilde{\varphi}|_{n+1}+\phi$, then the discrepancies read $\varepsilon[\varphi ,\tilde{Q}|_{n}] = \varepsilon[\tilde{\varphi}|_{n+1},\tilde{Q}|_{n}]+d\phi$ and the integrand of $\mathcal{J}[\varphi,\tilde{Q}|_{n}]$ becomes
\begin{equation}
\left\langle\varepsilon_j [\varphi ,\tilde{Q}|_{n}] , \varepsilon^j [\varphi ,\tilde{Q}|_{n}] \right\rangle=\left\langle\tilde{\varepsilon}_j[\tilde{\varphi}|_{n+1},\tilde{Q}|_{n}] , \tilde{\varepsilon}^j [\tilde{\varphi}|_{n+1},\tilde{Q}|_{n}]\right\rangle + \langle d\phi_j , d\phi^j \rangle +2\left\langle\ d\phi_j , \tilde{\varepsilon}^j [\tilde{\varphi}|_{n+1},\tilde{Q}|_{n}]\right\rangle
\end{equation}
Using the identity \eqref{eqApp_coclosed}, the third  term of this expansion reads $\left\langle\ d\phi_j , \tilde{\varepsilon}^j [\tilde{\varphi}|_{n+1},\tilde{Q}|_{n}]\right\rangle = d\left(\phi_j \star \tilde{\varepsilon}^j [\tilde{\varphi}|_{n+1},\tilde{Q}|_{n}]\right)$. As a consequence, integration of $\left\langle\varepsilon_j [\varphi ,\tilde{Q}|_{n}] , \varepsilon^j [\varphi ,\tilde{Q}|_{n}] \right\rangle$ over $\Omega$ offers
\begin{equation}
\mathcal{J}[\varphi ,\tilde{Q}|_{n}] =\mathcal{J}[\tilde{\varphi}|_{n+1},\tilde{Q}|_{n}] + \frac{1}{2}\int_{\Omega}{\langle d\phi_j , d\phi^j \rangle} +\int_{\partial\Omega}{\phi_j \star \tilde{\varepsilon}^j [\tilde{\varphi}|_{n+1},\tilde{Q}|_{n}]}
\end{equation}
Since the boundary conditions that $\mathrm{div}$ imposes in \eqref{eqApp_phiIter} are the ones of \eqref{eqAp_ELphi}, $\left.(d\tilde{\varphi}|_{n+1})\right|_{\partial\Omega} = \left.\left(\alpha\tilde{Q}|_n\right)\right|_{\partial\Omega}$ or equivalently, $\left.\tilde{\varepsilon} [\tilde{\varphi}|_{n+1},\tilde{Q}|_{n}]\right|_{\partial\Omega} = 0$, thus offering
\begin{equation}
\mathcal{J}[\varphi ,\tilde{Q}|_{n}] =\mathcal{J}[\tilde{\varphi}|_{n+1},\tilde{Q}|_{n}] +\frac{1}{2} \int_{\Omega}{\langle d\phi_j , d\phi^j \rangle} \Leftrightarrow \mathcal{J}[\varphi ,\tilde{Q}|_{n}] >\mathcal{J}[\tilde{\varphi}|_{n+1},\tilde{Q}|_{n}] 
\end{equation}

So $\tilde{\varphi}|_{n+1}$, obtained by resolving the Neumann boundary problem of the Poisson - like equation \eqref{eqApp_phiIter}, is the unique minimizer of $\mathcal{J}[\varphi,\tilde{Q}|_{n}] $, thus implying that
\begin{equation}
\label{eqApp_phiConv}
\mathcal{J}[\tilde{\varphi}|_{n+1} ,\tilde{Q}|_{n}] -\mathcal{J}[\tilde{\varphi}|_{n},\tilde{Q}|_{n}] =-\frac{1}{2}\left\|d(\tilde{\varphi}|_{n+1} -\tilde{\varphi}|_{n})\right\|^2
\end{equation}

Consequently, the considered iterative scheme of \eqref{eq_ELiter} leads to the convergence of the error functional $\mathcal{J}[\varphi,Q]$ to its global minimum and $\tilde{\varphi}|_{n} ,\tilde{Q}|_{n}$ converge in the equivalence class $[\tilde{\varphi}|_\infty,\tilde{Q}|_\infty]$.
%
%

\section{Additional experimental results}
\label{appD}
\subsection{Results on synthetic data sets}
\textit{Klein's bottle}: Parametrs $(u,v)\in[0,2\pi)$, equations of the bottle figure in $\mathbb{R}^3$:
\begin{subequations}
\begin{align*}
x(u,v)&=6\cos(u)(1+\sin(u))+4(1-0.5\sin(u))(\cos(u)\cos(v)\chi(u\leq\pi)+\cos(v+\pi)\chi(u>\pi))
\\
y(u,v)&=16\sin(u)+4(1-0.5\cos(u))\sin(u)\cos(v)\chi(u\leq\pi)
\\
z(u,v)&=4(1-0.5\cos(u))\sin(v)
\end{align*}
\end{subequations}
Fig. \ref{fig_klein} contains the actual embeddings of the Klein's bottle in $\mathbb{R}^3$ for the top-4 methods, according to the metrics of table \ref{tab_swiss}. The proposed method produces the closer results to IsoMap, mainly distorting the region of the  self-intersection (which cannot be represented by a $C^2$ embedding as the one of the proposed method). Being a global method, Diffusion Maps faithfully reconstructs the bottle figure, however up to global affine transformations, thus  breaking the distances to geodesics correspondence. Global affine transformations also distort the embedding of Laplacian Eigenmaps that  also suffers from local -- to -- global geometric consistency, since different regions deform in different ways.
\begin{figure}
  \centering
 \includegraphics[width=1.0\linewidth]{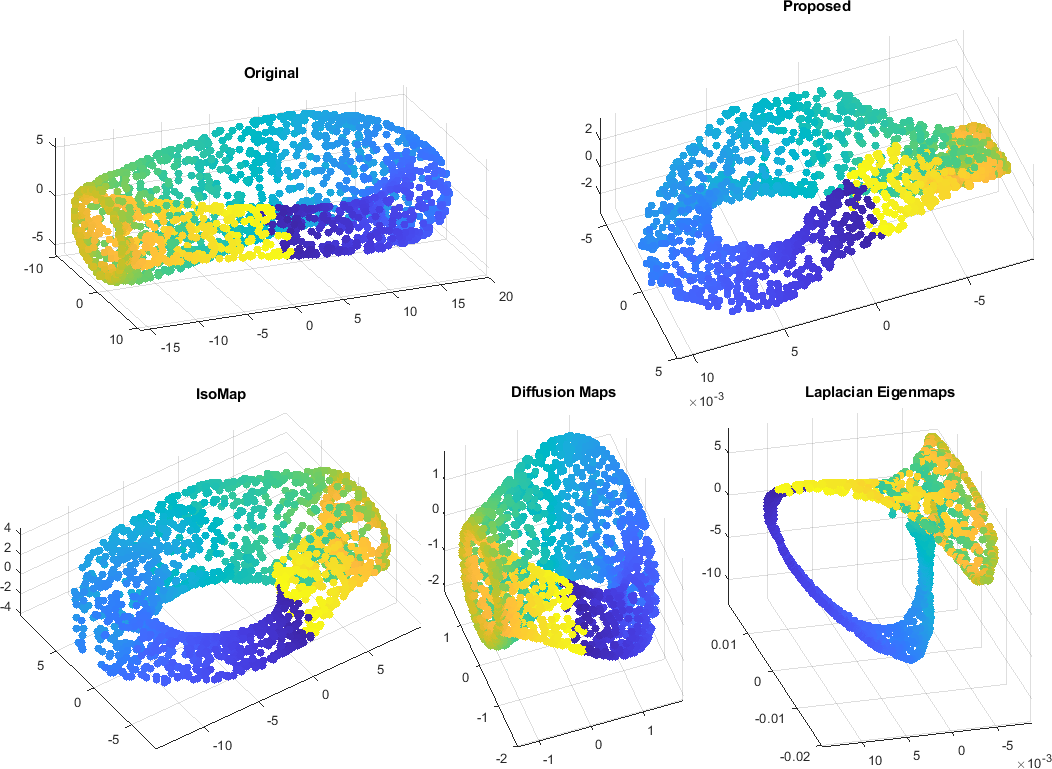}
  \caption{Reconstruction results on the Klein's bottle. Reconstruction of the original bottle figure from the neighborhoods' graph of constant intrinsic dimension 2, across all neighborhoods. The coloring depicts the data points' polar angle in the original figure. The embedding of IsoMap is used as ground truth for the approximate isometry of the original bottle figure. The global embedding of Diffusion Maps best approximates it up to affine transformations, while  the proposed local method follows, however up to orthogonal transformations (see also table \ref{tab_swiss}). As a result, the proposed local method follows the global embedding of IsoMap in the Euclidean approximation of the original manifold's geodesic distances, while Diffusion Maps distort them.}
\label{fig_klein}
\end{figure}

\textit{Flat torus}: Parametrs $(u,v)\in[0,2\pi)$, equations of the flat torus in $\mathbb{R}^4$:
\begin{equation*}
r(u,v)=(2\cos(u) , 2\sin(u) , \cos(v) , \sin(v))
\end{equation*}
equations of the short embedding in $\mathbb{R}^3$
\begin{subequations}
\begin{align*}
x(u,v)&=\cos(u)(2+\sin(v))
\\
y(u,v)&=\sin(u)(2+\sin(v))
\\
z(u,v)&=\cos(v)
\end{align*}
\end{subequations}
\begin{figure}
  \centering
 \includegraphics[width=1.0\linewidth]{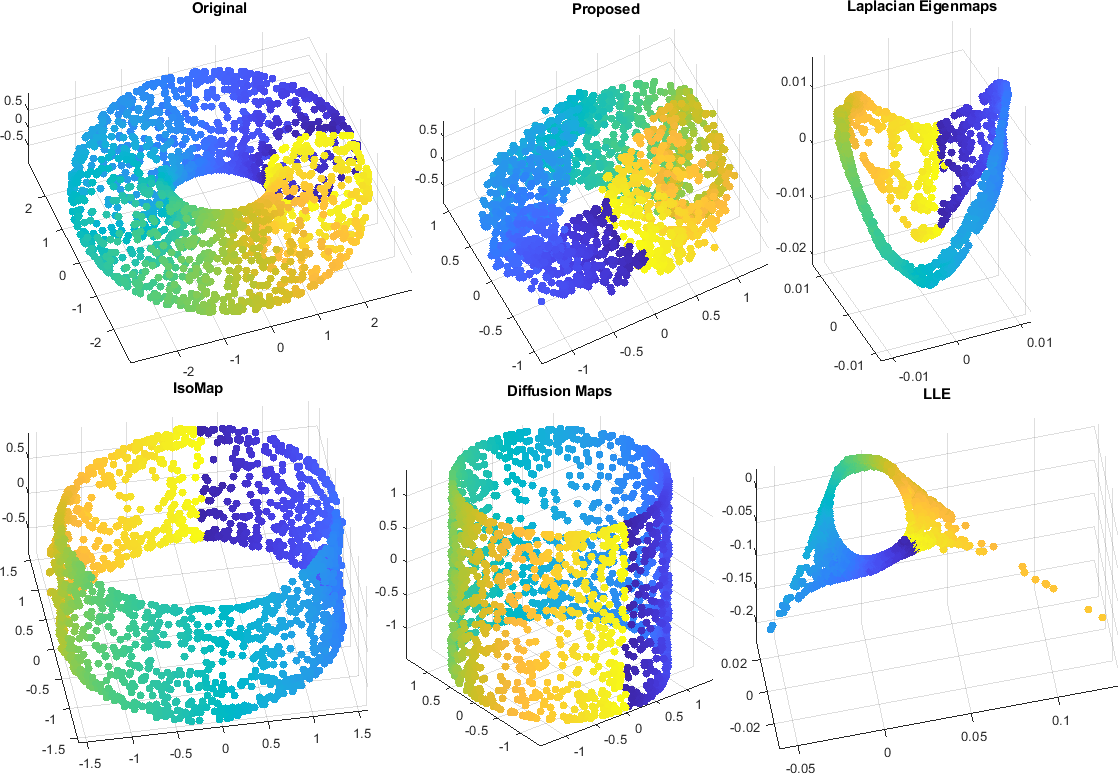}
  \caption{Embedding of the flat torus in $\mathbb{R}^3$. Reconstruction of this short embedding from the neighborhoods' graph of constant intrinsic dimension 2, across all neighborhoods. The coloring depicts the data points' polar angle in the original figure. The $(u,v)$ - parameterization of the manifold is used as ground truth for evaluating the embeddings' capability to recover the manifold's product structure. The global embedding of Diffusion Maps best approximates it up to affine transformations, while  the proposed local method follows, however up to orthogonal transformations (see also table \ref{tab_swiss}). Visually, the proposed embedding method best approximates the support of the $\mathbb{R}^3$ isometry of the flat torus, while the baseline ones, either local or global, they flatten the $v$-coordinate.}
\label{fig_klein}
\end{figure}
\subsection{Diagnoses' discrimination in the RNA-seq dataset}
\begin{figure}
  \centering
 \includegraphics[width=1.0\linewidth]{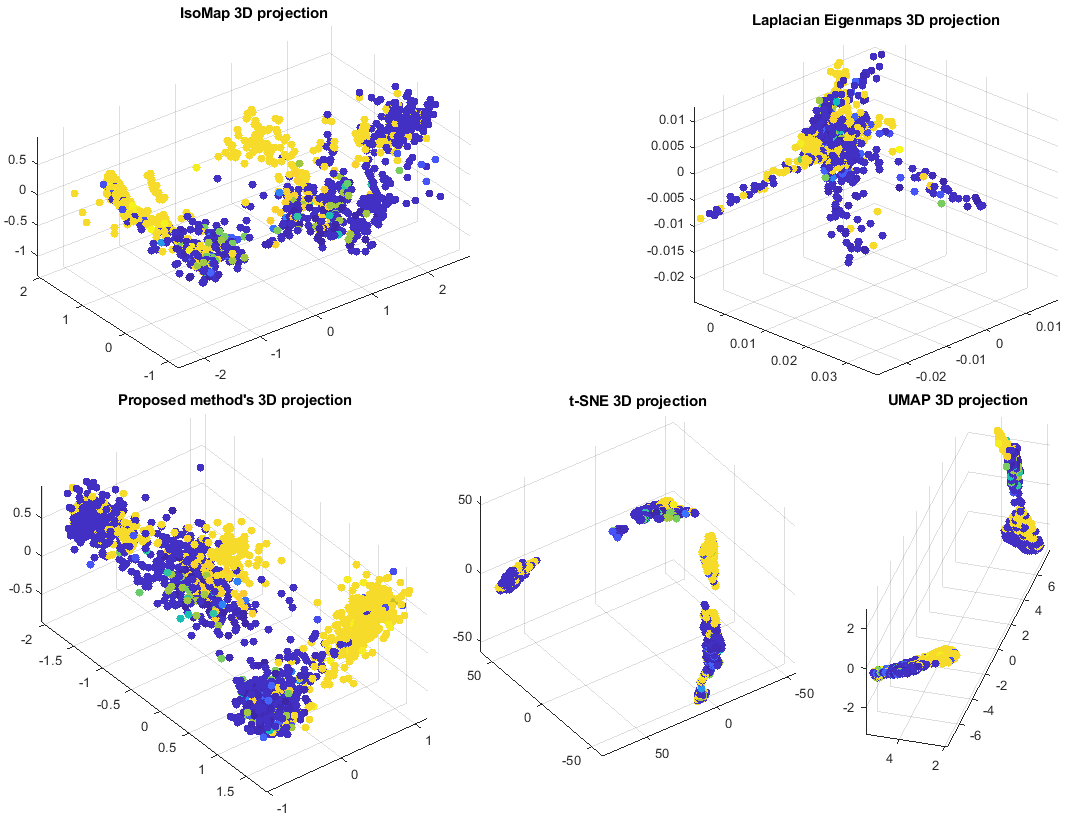}
  \caption{Visualization of the RNA-seq emedding. 3D projection of the computed 30-dimensional embeddings. Coloring is according to the lung cancer diagnosis. The proposed algorithm respects the global distances recovered by IsoMap and enhanses the data discrimination (table \ref{tab_pbmc}) by retaining  logal structures, up to integrability.}
\label{fig_pbmc}
\end{figure}
The connection between embeddins' local/global geometric consistency and their data discrimination capability, hypothesized by the results of table \ref{tab_pbmc}, becomes more evident via embeddings' visualization. In Fig. \ref{fig_pbmc}, the embeddings' three top-scaled dimensions are plotted and it becomes evident how the different methods distribute the datapoints, within their embeddings. The stochastic distribution -- based tSNE and UMAP practically cluster data according  to the cross-homogeneity of the neighborhoods' internal distribution. IsoMap and the proposed method  offer distances preserving embeddings that respectively enforce continuity and smoothness, thus not amplifying data clustering, however retaining the original neighboring relations. The deviation in the picture that IsoMap and the proposed method offer for the dataset mainly reflects the fact that data points' alignment in IsoMap is determined by the original points' shortest paths, while the alignment in the proposed method's picture is determined by the Laplacian inversion -- induced integration. As a result, any shortest paths' on-graph collapse is also evident in the IsoMap embedding, while Laplacian inversion imposes $C^2$ continuity thus not allowing such collapses. 

\subsection{Results on the FMNIST}
\begin{figure}
\centering
 \includegraphics[width=1\linewidth]{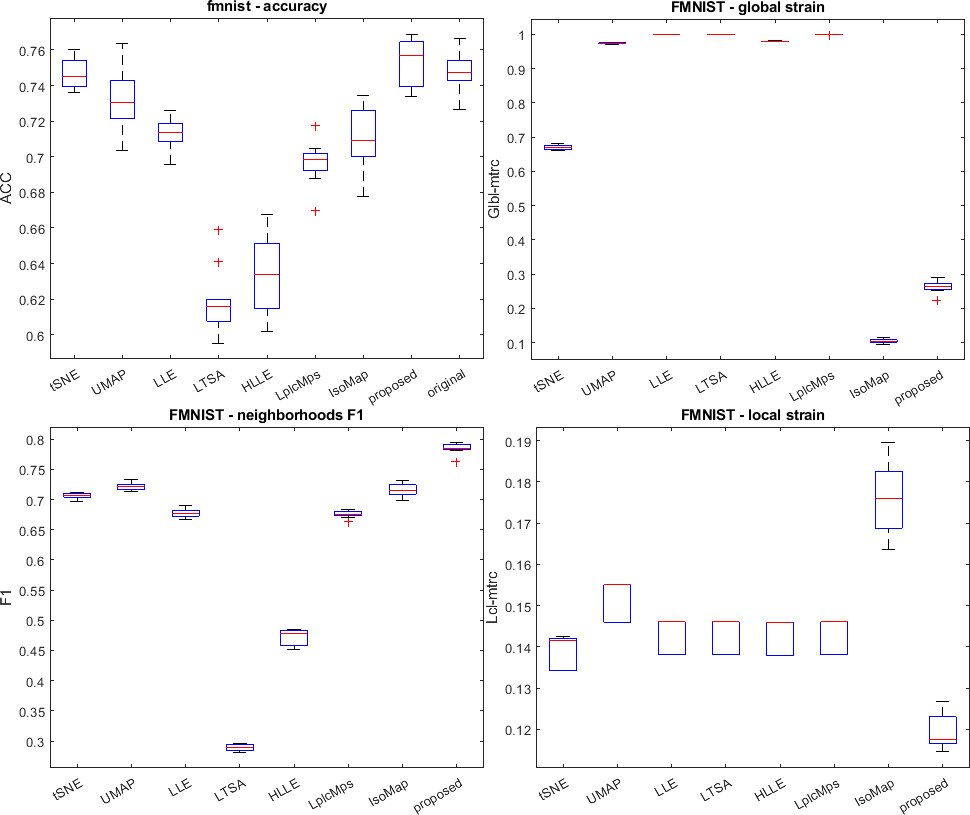}
  \caption{FMNIST dataset 10-fold experiment. The measures' box plots depict the maximum/minimum boundaries at whiskers, the boxes of the populations' 25th-75th percantile and the median values in red. Possible outliers are marked by red crosses. The proposed method best preserves neighboring relations and together with t-SNE, it best represents datapoints' clusters, slightly enhancing the clustering according to the "original" graph's neighborhoods. The performance of UMAP and IsoMap that preserve neighborhoods but not local distances, indicates that the original neighborhoods contain data points of mixed labels, however ordered according to data points' distances}
\label{fig_fmnist}
\end{figure}



\end{document}